%% file: paper.tex
\documentclass[11pt,a4paper]{article}
\usepackage[margin=1in]{geometry}

\usepackage[natbibapa]{apacite}
\usepackage{hyperref}
\usepackage{amsmath}
\usepackage{graphicx}
\usepackage{esvect}
\usepackage{amssymb}
\numberwithin{equation}{section}
\usepackage{float}
\usepackage{todonotes}
\usepackage{tabularx}
\usepackage{rotating}
\usepackage{booktabs}
\usepackage{setspace}
\usepackage{wasysym}
\usepackage{multirow}
\usepackage{threeparttable}
\usepackage{subcaption}

\usepackage[ruled,vlined]{algorithm2e}
\usepackage{float}
\usepackage{longtable}  

\usepackage{comment}

\usepackage{caption} 
\usepackage{makecell} 

\usepackage{array} 
\newcolumntype{P}[1]{>{\centering\arraybackslash}p{#1}}
\newcolumntype{M}[1]{>{\centering\arraybackslash}m{#1}}

\usepackage{tikz}
\usetikzlibrary{positioning, arrows.meta, shapes, babel}

\usepackage[T1]{fontenc}
\usepackage[utf8]{inputenc}
\usepackage[english]{babel}

\graphicspath{ {./images/} }

\newtheorem{theorem}{Theorem}[section]
\newtheorem{lemma}[theorem]{Lemma}

\newtheorem{proof}{Proof}[section]
\newtheorem{definition}{Definition}[section]

\title{The Multi-AMR Buffer Storage, Retrieval, and Reshuffling Problem: Exact and Heuristic Approaches}

\author{
    \underline{Max Disselnmeyer \footnote{Karlsruhe Institute of Technology, Zirkel 2, 76131 Karlsruhe, Germany, max.disselnmeyer@kit.edu, ORCID: \url{https://orcid.org/0009-0008-5689-2235}}},
    Thomas Bömer\footnote{Karlsruhe Institute of Technology, Zirkel 2, 76131 Karlsruhe, Germany, thomas.boemer@kit.edu, ORCID: \url{https://orcid.org/0000-0003-4979-7455}}, 
    Laura Dörr\footnote{Karlsruhe Institute of Technology, Zirkel 2, 76131 Karlsruhe, Germany, laura.doerr@kit.edu, ORCID: \url{https://orcid.org/0000-0002-8007-1815}},
    Bastian Amberg \footnote{Karlsruhe Institute of Technology, Zirkel 2, 76131 Karlsruhe, Germany, bastian.amberg@kit.edu, ORCID: \url{https://orcid.org/0000-0001-6715-3819}},
    Anne Meyer\footnote{Karlsruhe Institute of Technology, Zirkel 2, 76131 Karlsruhe, Germany, anne.meyer@kit.edu, ORCID: \url{https://orcid.org/0000-0001-6380-1348}}
}
\date{}

\begin{document}
\maketitle

\begin{abstract}
Buffer zones are essential in production systems to decouple sequential processes. In dense floor storage environments, such as space-constrained brownfield facilities, manual operation is increasingly challenged by severe labor shortages and rising operational costs. Automating these zones requires solving the Buffer Storage, Retrieval, and Reshuffling Problem (BSRRP). While previous work has addressed scenarios, where the focus is limited to reshuffling and retrieving a fixed set of items, real-world manufacturing necessitates an adaptive approach that also incorporates arriving unit loads. This paper introduces the Multi-AMR BSRRP, coordinating a robot fleet to manage concurrent reshuffling, alongside time-windowed storage and retrieval tasks, within a shared floor area. We formulate a Binary Integer Programming (IP) model to obtain exact solutions for benchmarking purposes. As the problem is NP-hard, rendering exact methods computationally intractable for industrial scales, we propose a hierarchical heuristic. This approach decomposes the problem into an $A^{*}$ search for task-level sequence planning of unit load placements, and Constraint Programming (CP) approach for multi-robot coordination and scheduling. Experiments demonstrate orders-of-magnitude computation time reductions compared to the exact formulation. These results confirm the heuristic’s viability as responsive control logic for high-density production environments.
\end{abstract}


\newpage
\tableofcontents

\include{sections/introduction}

\include{sections/literature}

\include{sections/problem_description}

\include{sections/ip_model}

\include{sections/heuristic}

\include{sections/experiments}

\include{sections/conclusion}

\include{sections/acknowledgments}

\newpage
\bibliographystyle{apacite}
\bibliography{literature}

\appendix
\newpage

\include{sections/appendix_model}
\include{sections/nphardness}

\end{document}

%% file: sections/introduction.tex
\section{Introduction}
\label{sec:intro}

The demand for automation in production supply and intralogistics, especially for autonomous mobile robots (AMR), is growing to optimize material flow and address critical labor shortages \citep{su13042075, descartes_study_reveals_2023}. These challenges are amplified in brownfield manufacturing facilities, which—despite benefits regarding sustainability and land availability—present hurdles such as spatial limitations and legacy infrastructure \citep{andulkar2018multi}. These constraints force the use of dense storage in the form of deep floor-level block layouts, where accessibility is restricted. AMRs offer greater flexibility than traditional Automated Guided Vehicles (AGVs) or forklifts, significantly improving material flow and productivity \citep{FRAGAPANE2021405}. According to \citep{amr_statistic},  the AMR market is expected to grow by $14.4\%$ from 2026 to 2033.

Efficient management of buffer zones (temporary storage areas decoupling production stages) is a critical challenge in these space-constrained environments. In manual operation, these zones suffer from labor shortages, operational inefficiencies and time lost searching for materials. A representative real-world example for automation from the surface coating industry is illustrated in Figure \ref{fig:real_world_usecase}, where a buffer zone must be integrated into the irregular residual space surrounding existing machinery.

\begin{figure}[!ht]
    \centering
    
    \begin{subfigure}[c]{0.65\textwidth}
        \centering
        \includegraphics[width=\textwidth]{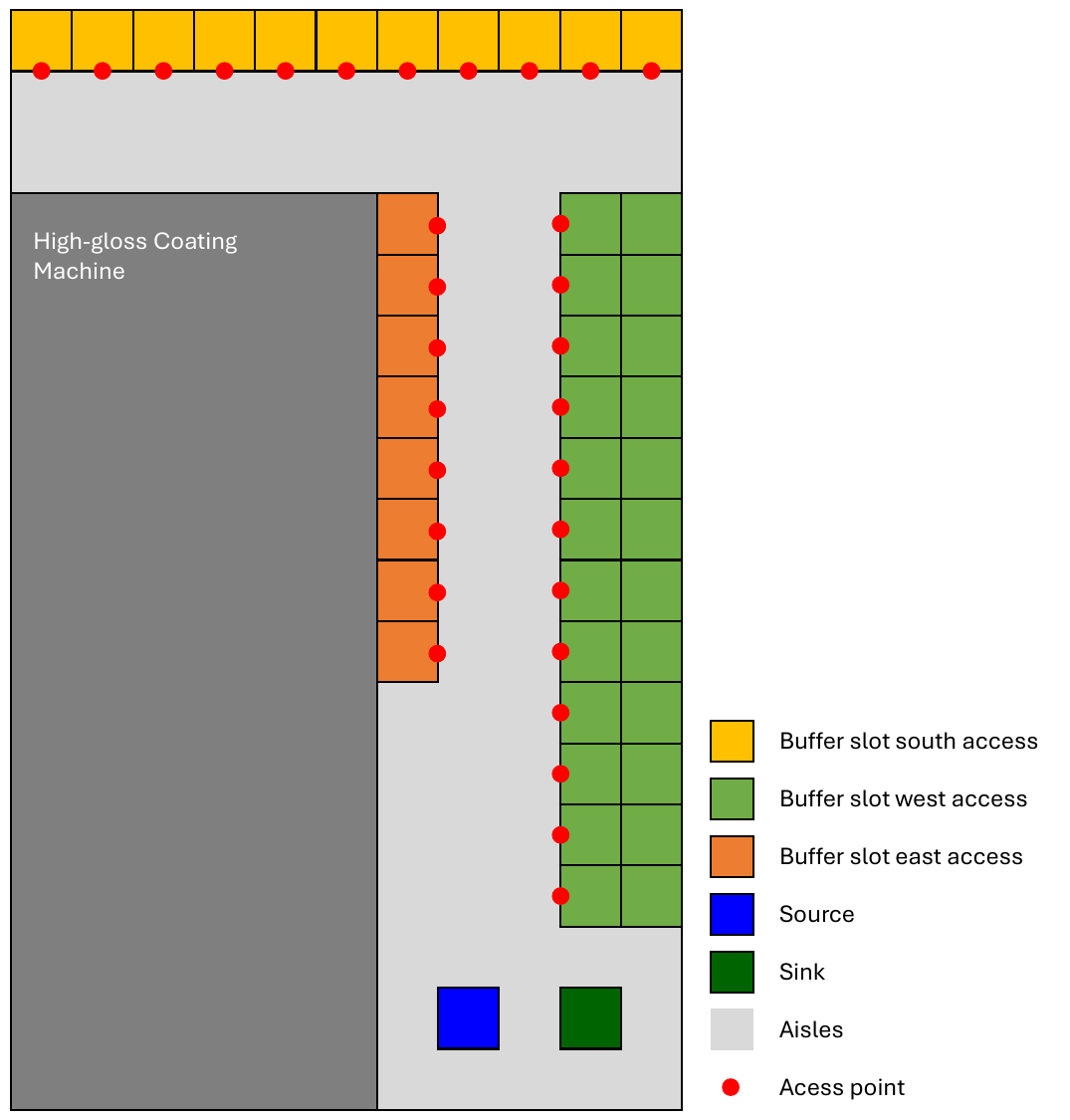}
        \caption{Schematic layout discretizing the residual space around a high-gloss coating machine into static LIFO storage lanes.}
        \label{fig:intro_layout}
    \end{subfigure}
    \hfill 
    \begin{minipage}[c]{0.3\textwidth}
        
        \begin{subfigure}{\textwidth}
            \centering
            \includegraphics[width=\textwidth]{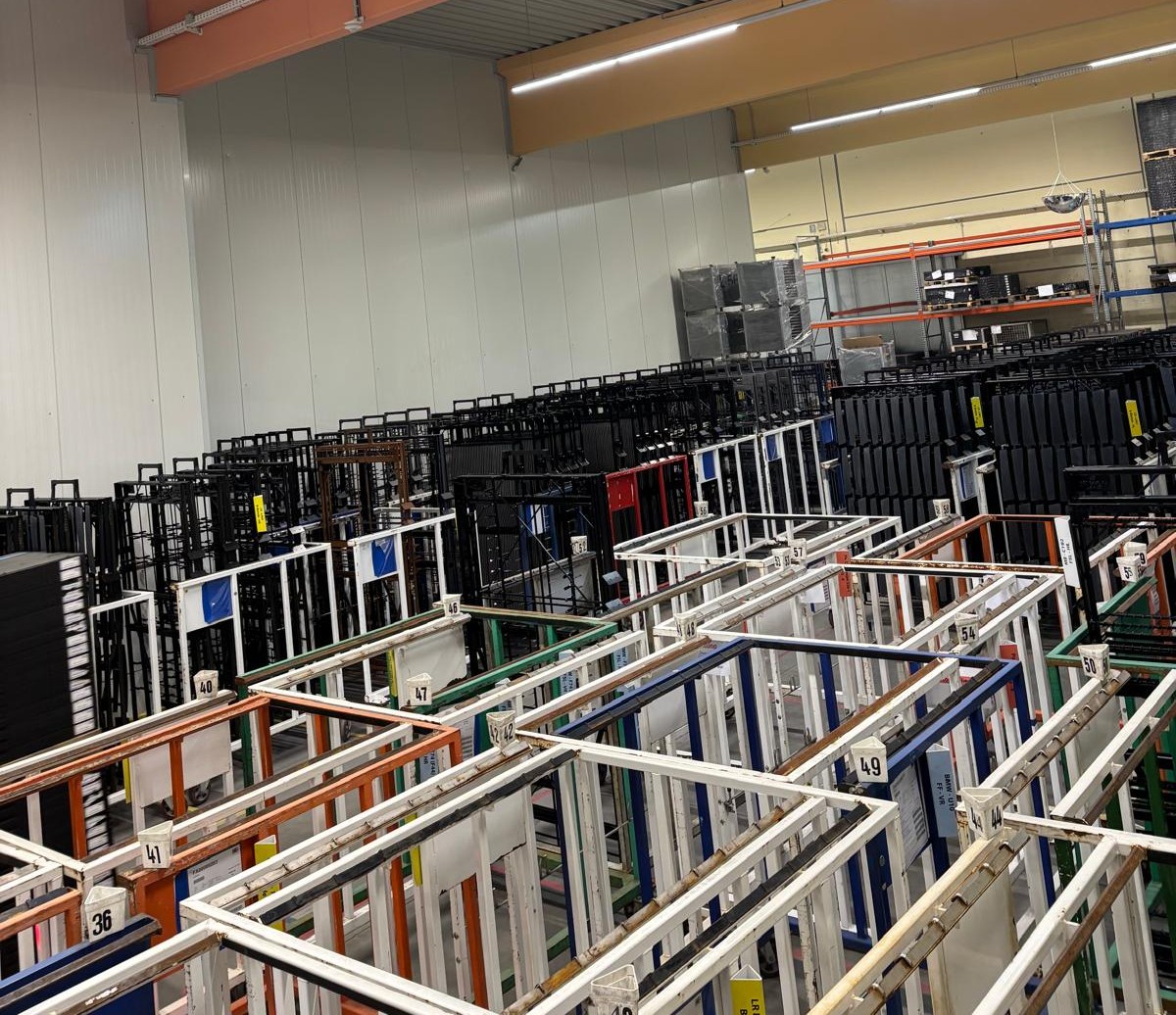}
            \caption{View from the source/sink area towards the coating machine (background wall on the left), highlighting the storage density.}
            \label{fig:intro_brownfield}
        \end{subfigure}
        
        \vspace{0.4cm} 
        
        \begin{subfigure}{\textwidth}
            \centering
            \includegraphics[width=\textwidth]{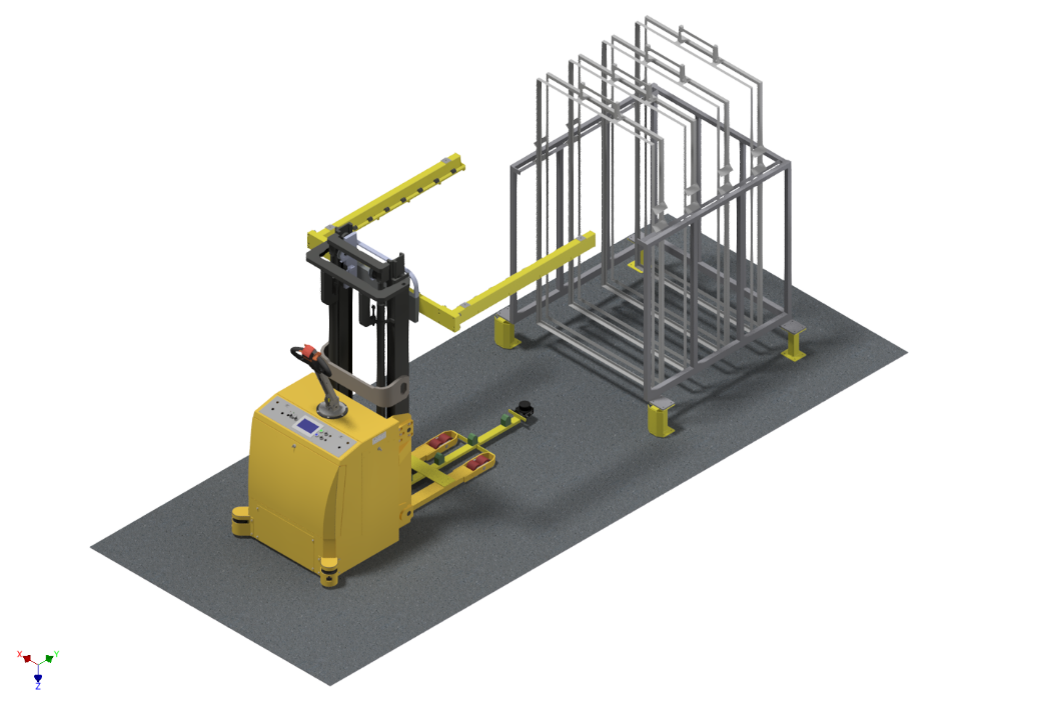}
            \caption{AMR designed to autonomously transport the specialized unit loads within constrained aisles.}
            \label{fig:intro_prototype}
        \end{subfigure}
        
    \end{minipage}
    
    \caption{Real-world buffer scenario from the surface coating industry. Automating material flow in such space-constrained brownfield facilities requires solving the Buffer Storage, Retrieval, and Reshuffling Problem (BSRRP) to ensure continuous machine supply.}
    \label{fig:real_world_usecase}
\end{figure}

Automating buffer zones with AMRs requires solving the Buffer Storage, Retrieval, and Reshuffling Problem (BSRRP). While previous research has established a model for buffer scenarios, where the focus is limited to retrieving and reshuffling a fixed set of unit loads \citep{disselnmeyer2024static}, real-world manufacturing environments also include storage. They involve an inflow of new unit loads and require the coordination of multiple AMRs within confined spaces, necessitating advanced optimization to avoid collisions and deadlocks while meeting strict retrieval deadlines.

This paper addresses these requirements by extending this retrieval-focused formulation to incorporate storage decisions and adapt it to the Multi-AMR setting. Our contributions are as follows:
\begin{itemize}
    \item \textbf{Exact Formulation (EF)} We develop a Binary Integer Programming (IP) model for the Multi-AMR BSRRP. It couples storage, retrieval and reshuffling decisions in dense storage, managing a fleet restricted to perimeter access like conventional forklifts. 
    \item \textbf{Complexity Analysis} We establish the computational complexity of the BSRRP problem, referencing a formal proof of NP-hardness via reduction from the Block Relocation Problem (BRP).
    \item \textbf{Hierarchical Heuristic} We propose a scalable hierarchical heuristic that combines an A* algorithm for creating storage, retrieval and reshuffling decisions with a Constraint Programming (CP) formulation for precise multi-AMR scheduling.
    \item \textbf{Rigorous Validation} We employ a rigorous validation methodology where heuristic solutions are injected into the exact IP model. This allows us to verify feasibility using a commercial solver and explicitly quantify the optimality gap, providing a ground-truth benchmark for the heuristic's performance.
\end{itemize}

The remainder of this paper is structured as follows: 
Section \ref{sec:literature} reviews the relevant literature on the Block Relocation Problem and multi-AMR coordination to situtate the BSRRP within the current research landscape. 
In Section \ref{sec:problem_description}, we provide a detailed description of the Multi-AMR BSRRP, including the logical subdivision of the storage space and the underlying operational assumptions.
Section \ref{sec:ip_model} presents the formal Binary Integer Programming formulation and discusses the computational complexity of the problem. 
The proposed hierarchical heuristic, combining $A^*$ with Constraint Programming, is detailed in Section \ref{sec:heuristic}. 
Section \ref{sec:computational_experiments} evaluates the performance of both the exact and heuristic approaches through extensive computational experiments and discusses managerial insights derived from the results. 
Finally, Section \ref{sec:conclusion} concludes the paper and outlines avenues for future research.

%% file: sections/literature.tex
\section{Related Work}
\label{sec:literature}
This section reviews the current state of research and defines the scientific context of the Multi-AMR BSRRP. First, we compare the problem to foundational approaches such as the Block Relocation Problem (BRP) and Multi-Agent Path Finding (MAPF). Second, we outline the specific modeling requirements for autonomous buffers operated by an AMR fleet, focusing on the transition from single AMR models to integrated multi-AMR systems. Third, we discuss existing heuristics for complex logistics tasks. Finally, we identify the research gaps regarding holistic, real-time control in fragmented environments, which forms the basis for the solutions presented in this paper.

\subsection{Foundations and Related Problems}
\label{subsec:foundations}

Container retrieval subject to strict LIFO constraints is formalized as the Block Relocation Problem (BRP) \citep{KIM2006940, LERSTEAU2022108529}. This problem is closely related to the Pre-Marshalling Problem (PMP) and its intralogistics variant, the Unit-Load Pre-Marshalling Problem (UPMP) \citep{pfrommer2023solving, pfrommer2024sorting, BOMER2025508}. Analogously, the steel industry addresses the Slab Pre-Marshalling Problem (SPMP), which is governed by identical LIFO stacking constraints \citep{Ge02072020}.

The BRP serves as a foundation to the Buffer Reshuffling and Retrieval Problem (BRR), which focuses on unit load relocation and retrieval in constrained spaces \citep{disselnmeyer2024static}. However, the Multi-AMR BSRRP presented in this paper introduces decisive additional challenges: the arrival and storage of new unit loads into the buffer, near real-time decision-making, and the coordination of multiple AMRs. Unlike PMP and UPMP, the BSRRP explicitly couples reshuffling and retrieval tasks with continuous fleet navigation in a shared workspace. While BRP research often assumes single-crane operations \citep{JI20151, TANG2010368}, the BSRRP utilizes AMRs with greater operational flexibility, enabling coordinated navigation within a shared workspace without rigid segmentation.

Furthermore, the BSRRP is distinct from other related optimization problems. It differs from the Yard Crane Scheduling Problem (YCSP) by avoiding crane-specific constraints like non-crossing \citep{kizilay2021comprehensive}. While space-constrained AGV scheduling \citep{Chen19052019} addresses system deadlocks through capacity-aware production scheduling, it typically lacks the active reshuffling logic required for deep LIFO stacks intrinsic to the BSRRP. Furthermore, the problem differs from the Parallel Stack Loading Problem (PSLP) by managing continuous evolution rather than just initial placement \citep{boge2020parallel}. Unlike the Storage Location Assignment Problem (SLAP), which provides static optimal snapshots \citep{reyes2019storage}, BSRRP manages ongoing reorganization. Additionally, while the Pallet Retrieval and Processing Problem (PRPP) optimizes the interface between transport and processing \citep{buckow2025integrated}, it neglects the complex internal reshuffling defined by BSRRP. Finally, unlike Robotic Mobile Fulfillment Systems (RMFS) which often deal with stochastic human picking times \citep{teck2024simulation}, BSRRP focuses on deterministic reshuffling to minimize travel distance.

\subsection{Extensions for Modeling the Buffer Storage, Retrival, and Reshuffling Problem}

This study extends previous research on the BRR problem \citep{disselnmeyer2024static} by incorporating storage operations into the buffer and multi-AMR management. We build upon the foundational BRP model by \cite{borjian2015managing} for storing, retrieving and reshuffling container stacks in in a single-crane yard, adapting it to address key requirements for autonomous buffer zones:

\begin{enumerate}
    \item \textbf{Objective Function:} Prioritizing the minimization of total AMR travel distance rather than the number of moves, which is the traditional focus of the BRP literature.
    \item \textbf{Unrestricted Relocation:} Allowing the relocation of any blocking unit load, not just those directly blocking the target, which is a common approach for the BRP.
    \item \textbf{Retrieval Time Windows:} Incorporating strict time windows for retrieval to ensure process synchronization.
    \item \textbf{Time-Based Modeling:} Utilizing discrete time steps for precise multi-AMR coordination and collision avoidance.
    \item \textbf{Storage Decisions:} Accommodating new unit loads requiring placement during ongoing operations, rather than focusing solely on reshuffling and retrieval tasks. 
\end{enumerate}

\subsection{Heuristic and Learning-Based Approaches}

Given the NP-hard nature of BRP variants, exact methods are often intractable for real-time decision-making, prompting the widespread use of heuristics \citep{KIM2006940, LERSTEAU2022108529}. However, the BSRRP extends beyond the purely combinatorial challenge of reshuffling: it necessitates the translation of moves into collision-free trajectories. Consequently, the problem encompasses both pathfinding, commonly solved using A* or Multi-Agent Path Finding (MAPF) techniques \citep{wang2024mapf}, and task allocation, which relates to the Vehicle Routing Problem (VRP) \citep{dantzigVRP1959, ARCHETTI2025}.

This complexity highlights the potential of hybrid decomposition approaches. For example, \cite{bomer2024sorting} successfully applied a sequential method combining A* search for reshuffling logic with a mixed-integer program for multi-AMR tour planning in the UPMP context. This demonstrates the viability of leveraging well-established heuristics to solve coupled subproblems sequentially. In the broader context of pre-marshalling, recent advances have also employed Monte Carlo Tree Search (MCTS) to effectively manage the cascading chain effects of reshuffling moves \citep{Wang02072024}.

Most recently, concurrent research has explicitly addressed the intersection of multi-agent pathfinding and movable obstacles in dense environments: \cite{hu2025conflictbasedsearchprioritizedplanning} introduced M-PAMO (MAPF Among Movable Obstacles), utilizing Conflict-Based Search (CBS) to resolve dependencies between agents and movable blockers. Addressing extreme density, \cite{makino2025mapfhdmultiagentpathfinding} proposed MAPF-HD (MAPF-for high density environments), utilizing swapping heuristics to manage obstructing agents. 

Finally, \cite{fu2026symbolicplanningmultiagentpath} formalized the Block Rearrangement Problem (BRaP) for dense storage grids. Their approach models the system as a discrete sliding-tile puzzle where agents navigate within the grid to rearrange blocks. Complementing this, \cite{geft2026robust} investigate the theoretical limits of relocation-free retrieval under uncertainty. They prove that relocations can be eliminated through robust storage arrangements if specific empty space ratios are maintained. 
While these works establish important foundations for dense storage, they assume internal grid accessibility or focus on static sequence optimization. In contrast, the BSRRP addresses scenarios where AMRs are restricted to perimeter access (see Figure \ref{fig:perimeter_access}), due to the characteristics of the layout, unit loads, or AMR fleet, necessitating logic to handle LIFO access constraints from the outside. Furthermore, unlike the focus on minimizing moves between static snapshots found in these approaches, our work addresses the continuous temporal synchronization required for floor-handling AMRs to meet strict time windows.

\input{figures/access}

\subsection{Research Gap}
Despite the extensive literature on BRP and autonomous logistics, a significant research gap remains regarding the Multi-AMR BSRRP. As summarized in Table \ref{tab:research_gap}, existing research tends to isolate specific subproblems, failing to address the interdependent complexities of modern brownfield intralogistics. Specifically, the table categorizes state-of-the-art approaches across key methodological dimensions: handling operations, access mode (physical constraints), fleet configuration, temporal representation and the Method used for solving the corresponding problem in that paper.

\begin{table}[!ht]
    \centering
    \caption{Comparison of related literature in the field of Block Relocation and Buffer Reshuffling.}
    \label{tab:research_gap}
    \resizebox{\textwidth}{!}{%
    \begin{tabular}{lcccccc}
        \toprule
        \textbf{Reference} & 
        \textbf{\thead{Handling\\Operations}} & 
        \textbf{\thead{Access\\Mode}} & 
        \textbf{\thead{Fleet}} & 
        \textbf{\thead{Temporal\\Representation}} & 
        \textbf{\thead{Method}} \\
        \midrule
        \cite{CASERTA201296} & $\LEFTcircle$ & \textbf{Perimeter} & 1 Crane & Moves & IP \\
        \cite{borjian2015managing} & \textbf{\CIRCLE} & \textbf{Perimeter} & 1 Crane & \textbf{Continuous} & IP \\
        \cite{JI20151} & $\LEFTcircle$ & \textbf{Perimeter} & n Cranes & Moves & Genetic Algo. \\
        \cite{bomer2024sorting} & $\bigcirc$ & \textbf{Perimeter} & \textbf{n AMRs} & Moves & CP \\
        \cite{disselnmeyer2024static} & $\LEFTcircle$ & \textbf{Perimeter} & 1 AMR & \textbf{Continuous} & IP \\
        \cite{hu2025conflictbasedsearchprioritizedplanning} & $\bigcirc$ & In-Grid & \textbf{n AMRs} & Steps & CBS Search \\
        \cite{fu2026symbolicplanningmultiagentpath} & \LEFTcircle & In-Grid & \textbf{n AMRs} & Steps & Symbolic Planning\\
        \midrule
        \textbf{This Paper} & \textbf{\CIRCLE} & \textbf{Perimeter} & \textbf{n AMRs} & \textbf{Continuous} & \textbf{IP + Heuristic} \\
        \bottomrule
    \end{tabular}}
    
    \vspace{0.3em}
    \footnotesize
    \begin{tabular}{@{}ll}
         \textbf{Operations:} & $\bigcirc$ Reshuffling only, $\LEFTcircle$ Reshuffling \& Retrieval, $\CIRCLE$ Storage, Reshuffling \& Retrieval \\
         \textbf{Temporal Rep.:} & \textit{Moves} (abstract sequence), \textit{Steps} (discrete synchronous), \textit{Continuous} (variable durations).
    \end{tabular}
\end{table}

Regarding the latter, we distinguish between three levels of abstraction that define when the system state is evaluated: 
\begin{itemize} 
    \item \textit{Moves} treat operations as a logical sequence; the system state is only updated between complete crane or robot movements, effectively ignoring durations. 
    \item \textit{Steps} partition time into fixed, synchronous intervals (ticks); the entire system state is recalculated at every interval (e.g., every 5 seconds), which is common in grid-based pathfinding but imposes rigid synchronization. 
    \item \textit{Continuous} representations allow for variable task durations. While implemented on a discrete time grid for computational feasibility, the model treats travel times as distance-dependent parameters rather than fixed steps, enabling precise synchronization with external deadlines. 
\end{itemize}

Based on this comparison, we identify three specific gaps in the current body of work:

\begin{enumerate}
    \item \textbf{Lack of Integrated Models:} 
    Most existing approaches address only fragments of the problem. For instance, the BRR model \citep{disselnmeyer2024static} effectively optimizes the reshuffling and retrieval of a fixed set of unit loads but neglects the storage of unit loads arriving from upstream processes. Conversely, literature on the dynamic BRP (e.g., \cite{borjian2015managing}) accounts for incoming containers but typically minimizes the number of moves for a single crane. This metric is insufficient for AMR fleets, where minimizing travel distance and execution time is critical to meeting strict retrieval deadlines in a spatially distributed buffer. Consequently, there is no model that integrates the conflicting objectives of storage, retrieval, and reshuffling into a single formulation.

    \item \textbf{Insufficient Multi-AMR Coordination for Perimeter Access:} 
    While the coordination of multiple agents is central to MAPF, scalability remains a hurdle in dense, interactive environments \citep{wang2024mapf, hua2024adaptivemapf}. Standard MAPF ignores the manipulation of unit loads and focuses solely on computing collision-free paths for the agents. Conversely, multi-crane BRP literature addresses manipulation but relies on zoning strategies or rail-bound constraints \citep{JI20151} that do not apply to flexible AMRs. Furthermore, recent concurrent studies on Block Rearrangement or Movable Obstacles \citep{fu2026symbolicplanningmultiagentpath, hu2025conflictbasedsearchprioritizedplanning} model the problem as a discrete sliding-tile puzzle where agents navigate within the grid (Grid-Based traffic control). This abstraction differs fundamentally from the BSRRP, where AMRs access the buffer from the perimeter. There is currently no model that couples the combinatorial complexity of BRP reshuffling with the coordinated traffic management of a multi-AMR fleet required to meet strict service time windows.
    
    \item \textbf{Absence of Fleet-Aware Heuristics:} 
    Addressing the BSRRP requires tightly integrating storage assignment, reshuffling logic, retrieval sequencing, and vehicle routing. Solving these problems in isolation is known to yield suboptimal results, as the interdependencies between subproblems are lost \citep{ELWAKIL2022105609}. While heuristics exist for different BRP variants (see the survey by \citep{LERSTEAU2022108529}), there is a lack of scalable approaches specifically designed to couple the combinatorial storage, retrieval and reshuffling decisions with the spatio-temporal constraints of multi-AMR fleet routing.
\end{enumerate}

%% file: figures/access.tex
\begin{figure}[htbp]
    \centering
\begin{tikzpicture}[
    box/.style={
        draw=none,
        fill=gray!60,
        minimum size=0.8cm,
        text=white,
        font=\sffamily\large\bfseries,
        align=center
    },
    arrow/.style={
        ->,
        >={Stealth[length=3mm, width=3mm]},
        line width=1.5pt,
        color=orange!90!black
    },
    pod/.pic={
        \filldraw[fill=orange, draw=black, thick, rounded corners=2pt] (-0.35,0) rectangle (0.35,0.25);
        \filldraw[fill=gray!50, draw=black, thin] (-0.2,0.25) rectangle (0.2,0.32);
        \fill[black] (-0.2,0) circle (0.08);
        \fill[black] (0.2,0) circle (0.08);
    },
    forklift/.pic={
        \filldraw[fill=orange, draw=black, thick] (-0.4,0) rectangle (0.2, 0.45);
        \draw[thick] (-0.4, 0.45) -- (-0.4, 0.7) -- (0.2, 0.7) -- (0.2, 0.45);
        \fill[black] (-0.25,0) circle (0.12);
        \fill[black] (0.1,0) circle (0.12);
        \draw[line width=2pt, black] (0.25, 0.05) -- (0.25, 0.8);
        \draw[line width=2pt, black] (0.25, 0.3) -- (0.6, 0.3);
    }
]

    \begin{scope}[shift={(0,0)}]
        \fill[white] (0,0) rectangle (3,3);

        \draw[gray!60, thin] (1,0) -- (1,3);
        \draw[gray!60, thin] (2,0) -- (2,3);
        \draw[gray!60, thin] (0,1) -- (3,1);
        \draw[gray!60, thin] (0,2) -- (3,2);
        \draw[gray!60, thin] (0,0) -- (3,0); 

        \draw[very thick, black] (0,0) -- (0,3) -- (3,3) -- (3,0);

        \draw[arrow] (0.5, -0.75) -- (0.5, 2.1);
        
        \draw[arrow] (1.5, -0.75) -- (1.5, 0.1);
        
        \draw[arrow] (2.5, -0.75) -- (2.5, 1.1);

        \node[box] at (1.5, 0.5) {01};
        \node[box] at (2.5, 1.5) {04};
        \node[box] at (2.5, 2.5) {07};
        
        \path (1.5, -1.3) pic {forklift};

        \node at (1.5, -1.8) {a) Access from the Perimeter};
    \end{scope}

    \begin{scope}[shift={(6,0)}]
        \fill[white] (0,0) rectangle (3,3);

        \draw[gray!60, thin] (1,0) -- (1,3);
        \draw[gray!60, thin] (2,0) -- (2,3);
        \draw[gray!60, thin] (0,1) -- (3,1);
        \draw[gray!60, thin] (0,2) -- (3,2);
        \draw[gray!60, thin] (0,0) -- (3,0); 

        \draw[very thick, black] (0,0) -- (0,3) -- (3,3) -- (3,0);

        \draw[arrow] (0.5, -0.75) -- (0.5, 2.5) -- (2.1, 2.5);
        \draw[arrow] (0.5, -0.75) -- (0.5, 1.5) -- (2.1, 1.5);
        \draw[arrow] (0.5, -0.75) -- (0.5, 0.5) -- (1.1, 0.5);
        
        \draw[arrow] (1.5, -0.75) -- (1.5, 0.1);
        
        \draw[arrow] (2.5, -0.75) -- (2.5, 1.1);

        \node[box] at (1.5, 0.5) {01};
        \node[box] at (2.5, 1.5) {04};
        \node[box] at (2.5, 2.5) {07};

        \path (1.5, -1.3) pic {pod};

        \node at (1.5, -1.8) {b) In-Grid Access};
    \end{scope}

\end{tikzpicture}
\caption{Conceptual comparison of buffer accessibility: a) Access from the Perimeter, as defined in the BSRRP. a) In-Grid Access, typical for Robotic Mobile Fulfillment Systems (RMFS) where agents navigate within the grid; In the former, AMRs are restricted to the exterior aisles, necessitating strategic reshuffling of obstructing unit loads to access deep LIFO slots.}
    \label{fig:perimeter_access}
\end{figure}

%% file: sections/problem_description.tex
\section{The Multi-AMR Buffer Storage, Retrieval, and Reshuffling Problem}
\label{sec:problem_description}

This section formally defines the Multi-AMR Buffer Storage, Retrieval, and Reshuffling problem (BSRRP). We describe the physical storage environment, introduce the concept of Static Lanes as a static graph overlay to manage accessibility, and detail the objective function used to coordinate the fleet.

\subsection{Static Lanes and Graph Overlay}
\label{sec:static_lanes}

The buffer consists of a continuous floor area where unit loads are stacked directly on the ground. To ensure accessibility and prevent deadlocks in this dense environment, we overlay a fixed graph structure onto the continuous space. Following the approach of \cite{pfrommer2022solving}, we partition the storage area into a set of Static Lanes, denoted as $\mathcal{I}$.

Figure \ref{fig:buffer_layout} illustrates this concept using a $5 \times 5$ buffer layout. As shown in Figure \ref{fig:grid_raw}, the floor is first discretized into a grid of storage slots. In a second step (Figure \ref{fig:grid_lanes}), these slots are grouped into Static Lanes based on their accessibility from the perimeter aisles.

\begin{figure}[!ht]
    \centering
    \begin{subfigure}[b]{0.45\textwidth}
        \centering
        \includegraphics[width=\textwidth]{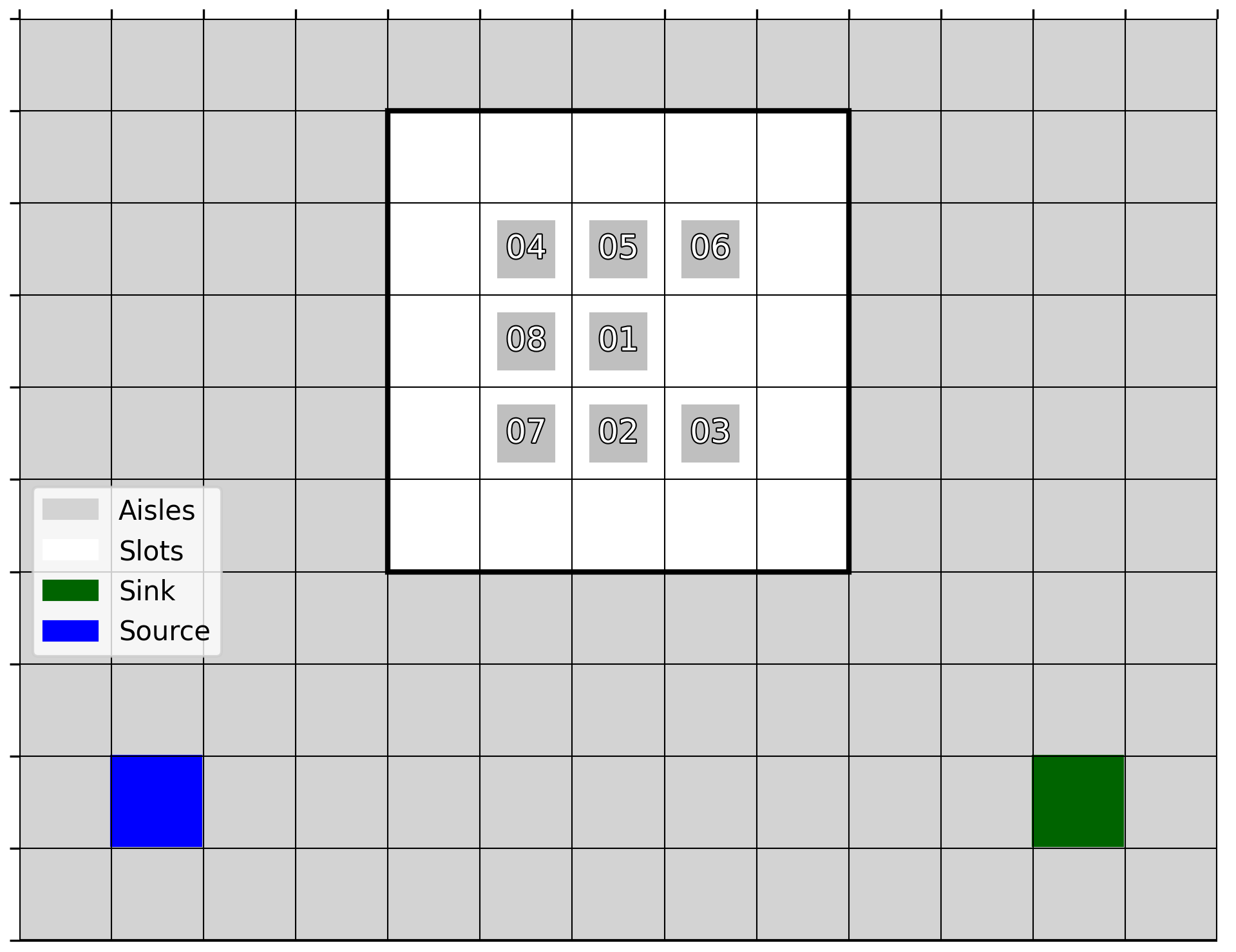} 
        \caption{Physical Grid Layout ($5 \times 5$)}
        \label{fig:grid_raw}
    \end{subfigure}
    \hfill
    \begin{subfigure}[b]{0.45\textwidth}
        \centering
        \includegraphics[width=\textwidth]{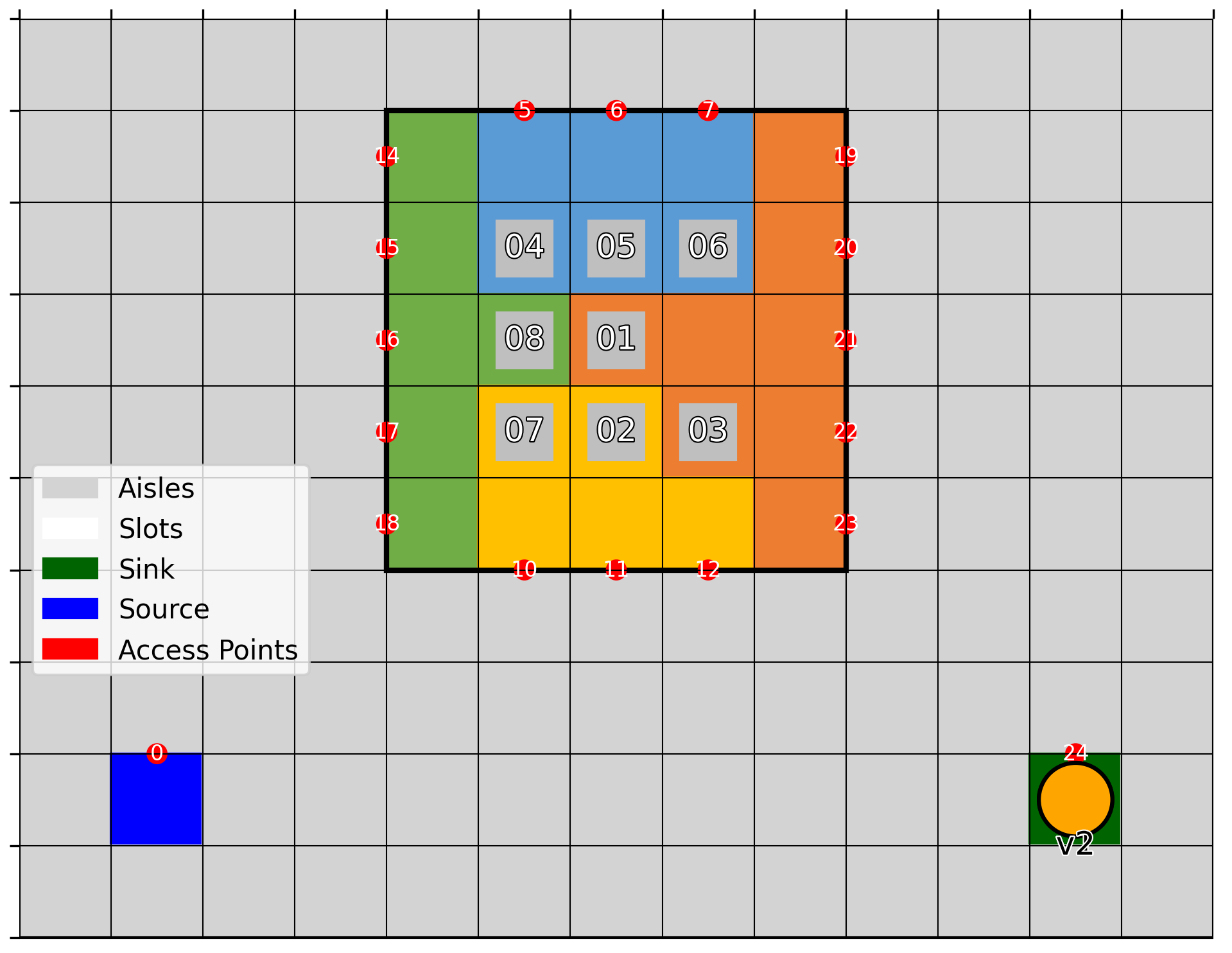}
        \caption{Static Lane Overlay}
        \label{fig:grid_lanes}
    \end{subfigure}
    \caption{Logical decomposition of the storage area. (a) The continuous floor is discretized into slots. (b) These slots are grouped into Static Lanes (colored), where each lane acts as a LIFO stack accessible from the perimeter. The grey numbered rectangles represent the unit loads}
    \label{fig:buffer_layout}
\end{figure}

This partitioning serves three critical operational purposes:

\begin{enumerate}
    \item \textbf{LIFO Constraints:} Each static lane $i \in \mathcal{I}$ functions as a Last-In-First-Out (LIFO) stack. As depicted in Figure \ref{fig:grid_lanes}, a lane consist of a sequence of floor slots. To retrieve a unit load stored deep within a lane $i$, all blocking unit loads placed in front of it (i.e., closer to the aisle) must first be reshuffled to empty slots in other lanes $k \in \mathcal{I} \setminus \{i\}$.
    
    \item \textbf{Deadlock Prevention:} A major challenge in multi-AMR systems on dense grids is congestion. To prevent circular deadlocks, we enforce a strict resource locking mechanism: Only one AMR may enter a lane at any given time. If multiple AMRs were to enter a single lane $i$ simultaneously, the leading robot would be blocked by the following one, causing a deadlock. Consequently, a lane $i \in \mathcal{I}$ is locked by an AMR during entry and exit operations.
    
    \item \textbf{Static Lane Configuration:} While the allocation of unit loads to slots is dynamic, the lane layout itself (the set $\mathcal{I}$) is static. For the instances analyzed in this study, the lane configuration is pre-computed (e.g., using a maximum-flow network formulation \citep{pfrommer2022solving}) to maximize storage density while ensuring connectivity. Once generated, this layout remains fixed throughout the runtime.
\end{enumerate}

\subsection{Model Simplifications}
To maintain computational tractability while capturing the core dynamics of the system, we apply the following abstractions:
\begin{itemize}
    \item \textbf{Kinematics:} We assume constant AMR velocity and negligible acceleration/deceleration phases. This allows us to model travel time as a linear function of distance on the graph.
    \item \textbf{Handling Times:} Unit load handling (picking up and setting down) is modeled with constant duration, assuming standardized load carriers.
    \item \textbf{Idealized Environment:} We assume a deterministic environment without disruptions (e.g., human traffic or breakdowns), focusing the optimization on the logical coordination of the fleet.
\end{itemize}

\subsection{Objective Function and Model Flexibility}
The primary objective of the BSRRP is to minimize the total distance traveled by the AMR fleet while strictly satisfying all service time windows. 

Given the set of unit loads $\mathcal{N}$ designated for retrieval, where each load $n \in \mathcal{N}$ has a release time $r_n$ and a deadline $d_n$ (derived from the retrieval window), the solver must determine a sequence of storage, reshuffling, and retrieval moves such that:
\begin{enumerate}
    \item Every retrieval job is completed within its time window $[r_n, d_n]$.
    \item No static lane constraints (LIFO, Capacity, Exclusive Access) are violated.
    \item The sum of distances for all loaded and empty AMR moves is minimized.
\end{enumerate}

The cumulative distance serves as a robust proxy for efficiency, reducing energy consumption and hardware wear. Furthermore, by strictly enforcing time window constraints, the model acts as a validation tool for strategic planning: if the operational requirements exceed the fleet's capacity given the layout, the model returns an infeasibility status, signaling the need for layout adjustments or fleet expansion.

%% file: sections/ip_model.tex
\section{Exact Problem Formulation}
\label{sec:ip_model}

We formulate the Multi-AMR Buffer Storage, Retrieval, and Reshuffling Problem (BSRRP) as a binary integer program, referred to as the Exact Formulation (EF). This model builds directly upon the Buffer Reshuffling and Retrieval formulation introduced in \citep{disselnmeyer2024static}. We extend this model to a multi-AMR setting by incorporating unit load arrivals—conceptually following the dynamic BRP formulation by \citep{borjian2015managing}—while adding explicit constraints for collision-free fleet coordination and floor-level maneuvering.

\subsection{Notation}
Let $\mathcal{N} = \{1, \dots, N \}$ be the set of unit loads, $\mathcal{I} = \{0, \dots , I \}$ the set of lanes, and $\mathcal{T} = \{1, \dots , T\}$ the set of time steps. The fleet of AMRs is denoted by $\mathcal{V} = \{1, \dots, V\}$. A specific storage location is defined as $[i, j]$ for lane $i \in \mathcal{I}$ and depth position $j \in \mathcal{J}_{i} = \{1, \dots, J_i\}$. 

The planning horizon $T$ is calculated based on the latest arrival ($a_n$) and retrieval ($r_n$) deadlines and the length of the arrival ($\alpha_n$) and retrieval window ($\rho_n$):
\begin{equation}
    T = \max_{n,m \in \mathcal{N}} \{a_{n} + \alpha_{n}, r_{m} + \rho_{m}\}
\end{equation}
The parameter $\tau_{ijkl}$ represents the time cost to move between slots. It defaults to the distance $d_{ijkl}$, but includes handling time $h$ for loaded moves: $\tau_{ijkl} = \max(1, d_{ijkl} + 2h)$. We enforce a lower bound of 1 to ensure that every action consumes time. This prevents the solver from scheduling instantaneous zero-cost idle loops, ensuring that waiting explicitly advances the system state.

Table \ref{tab:sets_indices} summarizes the sets, indices, and parameters.

\begin{table}[ht]
\centering
\caption{Sets, indices, and parameters}
\label{tab:sets_indices}
\begin{tabular}{l l}
\toprule
\textbf{Notation} & \textbf{Description} \\
\midrule
$\mathcal{N}, \mathcal{V}, \mathcal{T}$ & Sets of unit loads, AMRs, and time steps. \\
$\mathcal{I}, \mathcal{I}'$ & Sets of all lanes (incl. I/O), Set of all buffer lanes (excl. I/O). \\
$[i,j]$ & Slot at lane $i$ and position $j$ ($j=1$ is deepest in the lane). \\
$d_{ijkl}$ & Distance between slots $[i,j]$ and $[k,l]$. \\
$\tau_{ijkl}$ & Travel and handling time (cost) between slots $[i,j]$ and $[k,l]$. \\
$h$ & Handling time for picking up or dropping a unit load. \\
$[a_n, a_n+\alpha_n]$ & Arrival time window for unit load $n$. \\
$[r_n, r_n+\rho_n]$ & Retrieval time window for unit load $n$. \\
\bottomrule
\end{tabular}
\end{table}

\subsection{Decision Variables}
We classify the binary decision variables into AMR actions and system state indicators. All variables are defined as 1 if the condition holds and 0 otherwise.

The decision-making process is modeled using two categories of binary variables. First, AMR decision variables determine the specific tasks started by each vehicle $v$ at time $t$. We define variables for reshufflings ($x$), retrievals ($y$), and the storage of new arrivals ($z$). Additionally, the variable $e$ explicitly models empty travelling to accurately track the positions of the AMRs.

Second, state variables are required to track the physical configuration of the system. The variable $b$ maps the inventory of unit loads to slots while $c$ tracks the spatio-temporal position of every AMR to prevent collisions. The auxiliary variables $s$ and $g$ monitor the completion status of storage and retrieval requests, respectively.

\begin{align}
\intertext{\textbf{AMR Decision Variables:} Control the movement and handling tasks of each vehicle $v$.}
x_{ijklntv} &= 1 \text{ if AMR } v \text{ relocates } n \text{ from } [i,j] \text{ to } [k,l] \text{ at } t, \label{eq:x_var} \\
& \quad \quad \forall i,k \in \mathcal{I'}, \forall j \in \mathcal{J}_i, \forall l \in \mathcal{J}_{k}, \forall n \in \mathcal{N}, \forall t \in \mathcal{T}, \forall v \in \mathcal{V} \nonumber \\
y_{ijntv} &= 1 \text{ if AMR } v \text{ retrieves } n \text{ from } [i,j] \text{ at } t, \label{eq:y_var} \\
& \quad \quad \forall i \in \mathcal{I}\setminus\{I\}, \forall j \in \mathcal{J}_i, \forall n \in \mathcal{N}, \forall t \in \mathcal{T}, \forall v \in \mathcal{V} \nonumber \\
z_{ijntv} &= 1 \text{ if AMR } v \text{ stores } n \text{ into } [i,j] \text{ at } t, \label{eq:z_var} \\
& \quad \quad \forall i \in \mathcal{I'}, \forall j \in \mathcal{J}_i, \forall n \in \mathcal{N}, \forall t \in \mathcal{T}, \forall v \in \mathcal{V} \nonumber \\
e_{ijkltv} &= 1 \text{ if AMR } v \text{ performs an empty drive from } [i,j] \to [k,l] \text{ at } t, \label{eq:e_var} \\
& \quad \quad \forall i,k \in \mathcal{I}, \forall j \in \mathcal{J}_i, \forall l \in \mathcal{J}_{k}, \forall t \in \mathcal{T}, \forall v \in \mathcal{V} \nonumber \\
\intertext{\textbf{State Variables:} Track the location and completion status of loads and vehicles.}
b_{ijnt} &= 1 \text{ if unit load } n \text{ is in slot } [i,j] \text{ at time } t, \label{eq:b_var} \\
& \quad \quad \forall i \in \mathcal{I'}, \forall j \in \mathcal{J}_i, \forall n \in \mathcal{N}, \forall t \in \mathcal{T} \nonumber \\
c_{ijtv} &= 1 \text{ if AMR } v \text{ is present at } [i,j] \text{ at time } t, \label{eq:c_var} \\
& \quad \quad \forall i \in \mathcal{I}, \forall j \in \mathcal{J}_i, \forall t \in \mathcal{T}, \forall v \in \mathcal{V} \nonumber \\
s_{nt} &= 1 \text{ if unit load } n \text{ is stored by time } t-1, \quad \forall n \in \mathcal{N}, \forall t \in \mathcal{T} \label{eq:s_var} \\
g_{nt} &= 1 \text{ if unit load } n \text{ is retrieved by time } t-1, \quad \forall n \in \mathcal{N}, \forall t \in \mathcal{T} \label{eq:g_var}
\end{align}

\subsection{Objective Function}
The objective is to minimize the total travel distance of the AMR fleet. Minimizing distance serves as a proxy for energy efficiency and reduces hardware wear. Furthermore, reducing unnecessary travel alleviates congestion, which indirectly aids the throughput of the system. Note that service deadlines are treated as hard constraints to guarantee service levels; therefore, the objective focuses purely on efficiency rather than tardiness penalties.

The function sums the distances for retrieval ($y$), storage ($z$), reshuffling ($x$), and empty travel ($e$):

\begin{equation}
    \min \sum_{t \in \mathcal{T}, v \in \mathcal{V}} \left[ 
    \sum_{n \in \mathcal{N}} \sum_{\substack{i \in \mathcal{I} \\ j \in \mathcal{J}_i}} ( y_{ijntv} d_{ijI1} + z_{ijntv} d_{01ij} ) 
    + \sum_{\substack{i,k \in \mathcal{I} \\ j \in \mathcal{J}_i \\ l \in \mathcal{J}_k}} d_{ijkl} ( e_{ijkltv} + \sum_{n \in \mathcal{N}} x_{ijklntv} ) 
    \right]
\end{equation}

\subsection{Constraints}
The feasible region is governed by flow conservation, physical consistency, time windows, and traffic rules.

\paragraph{Flow Conservation \& State Updates}
We explicitly model the system dynamics using four flow conservation constraints that link the discrete AMR actions to the state of the unit loads and vehicles. Constraints \eqref{eq:state_s} define the storage status $s_{nt}$. A unit load $n$ is considered stored at time $t$ if it was stored initially ($s_{n1}$) or if a storage action $z$ transporting it from the source has been completed. This summation explicitly accounts for the travel time $\tau_{01ij}$, ensuring the status only updates after the transport is finished. Constraints \eqref{eq:state_g} analogously track the retrieval status $g_{nt}$. This variable is updated to 1 if a retrieval action $y$ has been initiated for unit load $n$ at any previous time step. Unlike storage, this updates at the start of the action to prevent the load from being accessed again. Constraints \eqref{eq:b_update} govern the slot occupancy $b_{ijnt}$ for every buffer slot $[i,j]$. The state at time $t$ is determined by the state at $t-1$, plus any unit loads arriving via reshuffling ($x$) or new storage ($z$) after their respective travel times, minus any loads leaving the slot due to reshuffling ($x$) or retrieval ($y$). Finally, constraints \eqref{eq:c_update} ensure spatio-temporal continuity for the mobile robots. The AMR position variable $c_{ijtv}$ tracks whether the vehicle $v$ is in the slot $[i,j]$ at time $t$. The constraint updates the position by adding vehicles arriving from storage ($z$), retrieval ($y$), reshuffling ($x$), or empty travel ($e$) tasks—accounting for the specific travel duration $\tau$ of each—and subtracting vehicles that depart to initiate these tasks.

\begin{align}
    s_{nt} &= s_{n1} + \sum_{i \in \mathcal{I}'} \sum_{j \in \mathcal{J}_i} \sum_{t'=1}^{t-\tau_{01ij}} \sum_{v \in \mathcal{V}} z_{ijnt'v} \quad \forall n \in \mathcal{N}, \forall t \in \mathcal{T} \label{eq:state_s} \\
    g_{nt} &= \sum_{i \in \mathcal{I} \setminus \{I\}} \sum_{j \in \mathcal{J}_i} \sum_{t'=1}^{t-1} \sum_{v \in \mathcal{V}} y_{ijnt'v} \quad \forall n \in \mathcal{N}, \forall t \in \mathcal{T} \label{eq:state_g} \\
    b_{ijnt} &= b_{ijn(t-1)} + \sum_{v \in \mathcal{V}} \Bigg[ \sum_{k \in \mathcal{I}'} \sum_{l \in \mathcal{J}_k} \left( x_{klijn(t-\tau_{klij})v} - x_{ijkln(t-1)v} \right) \nonumber \\
    & \quad - y_{ijn(t-1)v} + z_{ijn(t-\tau_{01ij})v} \Bigg] \quad \forall i \in \mathcal{I'}, \forall j \in \mathcal{J}_i, \forall n \in \mathcal{N}, \forall t \in \mathcal{T} \setminus \{1\} \label{eq:b_update} \\
    c_{ijtv} &= c_{ij(t-1)v} + \sum_{n \in \mathcal{N}} z_{ijn(t-\tau_{01ij})v} - \sum_{n \in \mathcal{N}} y_{ijn(t-1)v} \nonumber \\
    & \quad + \sum_{k \in \mathcal{I}'} \sum_{l \in \mathcal{J}_k} \sum_{n \in \mathcal{N}} x_{klijn(t-\tau_{klij})v} + \sum_{k \in \mathcal{I}} \sum_{l \in \mathcal{J}_k} e_{klij(t-\tau_{klij})v} \nonumber \\
    & \quad - \sum_{k \in \mathcal{I}'} \sum_{l \in \mathcal{J}_k} \sum_{n \in \mathcal{N}} x_{ijkln(t-1)v} - \sum_{k \in \mathcal{I}} \sum_{l \in \mathcal{J}_k} e_{ijkl(t-1)v} \nonumber \\
    & \quad \forall i \in \mathcal{I'}, \forall j \in \mathcal{J}_i, \forall t \in \mathcal{T} \setminus \{1\}, \forall v \in \mathcal{V} \label{eq:c_update}
\end{align}
\noindent {\small \textit{Note:} Eq. (\ref{eq:c_update}) defines the motion for storage lanes $\mathcal{I}'$; analogous conservation constraints apply for the Source ($c_{01tv}$) and Sink ($c_{I1tv}$) nodes.}

\paragraph{Initialization}
To accurately simulate the system's evolution, we must define its starting state at $t=1$ based on the input instance. Constraints \eqref{eq:init_b} map the initial slot occupancy, setting the state variable $b_{ijn1}$ to 1 if unit load $n$ occupies slot $[i,j]$ at the start of the planning horizon. Constraints \eqref{eq:init_s} distinguish between inistially stored unit loads and future arrivals; the storage status $s_{n1}$ is initialized to 1 for loads already present in the buffer, and 0 for those arriving at later time steps. Finally, Constraints \eqref{eq:init_c} establish the initial AMR positions, assigning each AMR $v$ to its designated starting coordinates $[i,j]$ by setting the position variable $c_{ij1v}$ accordingly.

\begin{align}
    b_{ijn1} &= \begin{cases}
        1, & \text{if unit load } n \text{ starts in slot } [i,j] \\
        0, & \text{otherwise}
    \end{cases} \quad \forall i \in \mathcal{I'}, \forall j \in \mathcal{J}_i, \forall n \in \mathcal{N} \label{eq:init_b} \\
    s_{n1} &= \begin{cases}
        1, & \text{if unit load } n \text{ is initially stored in buffer} \\
        0, & \text{otherwise}
    \end{cases} \quad \forall n \in \mathcal{N} \label{eq:init_s} \\
    c_{ij1v} &= \begin{cases}
        1, & \text{if AMR } v \text{ starts in slot } [i,j] \\ 
        0, & \text{otherwise}
    \end{cases} \quad \forall i \in \mathcal{I}, \forall j \in \mathcal{J}_i, \forall v \in \mathcal{V} \label{eq:init_c}
\end{align}

\paragraph{System Consistency}
We spatial and operational integrity through a set of constraints governing capacity, lane structure, and object permanence. Constraints \eqref{eq:consistency_cap} limit the capacity of each storage slot $[i,j]$, ensuring it holds at most one unit load at any given time $t$. Constraints \eqref{eq:consistency_grav} mandate a dense storage policy to ensure gapless lane utilization; a unit load may only occupy the outer position $j+1$ if the adjacent inner position $j$ is also occupied. This ensures that the lanes are filled continuously, reflecting the physical constraints of floor block storage. Constraints \eqref{eq:item_avail} ensure that an AMR $v$ can only initiate a retrieval ($y$) or reshuffling ($x$) for a unit load $n$ from slot $[i,j]$ if that load is present ($b_{ijnt}=1$). Finally, Constraints \eqref{eq:presence} link the vehicle's actions to its physical location. The sum of all tasks—reshuffling ($x$), empty travel ($e$), and retrieval ($y$)—initiated by AMR $v$ at slot $[i,j]$ is bounded by the presence variable $c_{ijtv}$, ensuring a robot must be at a location to operate from it.

\begin{align}
    & \sum_{n \in \mathcal{N}} b_{ijnt} \leq 1 \quad \forall i \in \mathcal{I'}, \forall j \in \mathcal{J}_i, \forall t \in \mathcal{T} \label{eq:consistency_cap} \\
    & \sum_{n \in \mathcal{N}} b_{ijnt} \geq \sum_{n \in \mathcal{N}} b_{i(j+1)nt} \quad \forall i \in \mathcal{I'}, \forall j \in \mathcal{J}_i \setminus \{J_i\}, \forall t \in \mathcal{T} \label{eq:consistency_grav} \\
    & \sum_{v \in \mathcal{V}} \left( y_{ijntv} + \sum_{k \in \mathcal{I}'} \sum_{l \in \mathcal{J}_k} x_{ijklntv} \right) \leq b_{ijnt} \quad \forall i \in \mathcal{I'}, \forall j \in \mathcal{J}_i, \forall n \in \mathcal{N}, \forall t \in \mathcal{T} \label{eq:item_avail} \\
    & \sum_{k \in \mathcal{I}'} \sum_{l \in \mathcal{J}_k} \sum_{n \in \mathcal{N}} x_{ijklntv} + \sum_{k \in \mathcal{I}} \sum_{l \in \mathcal{J}_k} e_{ijkltv} \nonumber \\
    & + \sum_{n \in \mathcal{N}} y_{ijntv} \leq c_{ijtv} \quad \forall i \in \mathcal{I}', \forall j \in \mathcal{J}_i, \forall t \in \mathcal{T}, \forall v \in \mathcal{V} \label{eq:presence}
\end{align}
\noindent {\small \textit{Note:} Constraints (\ref{eq:presence}) restrict actions in buffer lanes; separate presence constraints govern the Source (for $z, y, e$) and Sink (for $e$). For brevity, we omit the explicit formulation and refer to the complete model in Appendix \ref{sec:appendix_model}.}

\paragraph{Time Windows}
We model time windows as hard constraints, rendering any solution that misses a deadline infeasible. Constraints \eqref{eq:tw_retrieval} mandate that every unit load $n$ is successfully retrieved exactly once within its designated window $[r_n, r_n+\rho_n]$. Since the deadline applies to the arrival at the Sink $[I,1]$, the valid start time for a retrieval action $y_{ijntv}$ from a specific slot $[i,j]$ is shifted earlier by the travel time $\tau_{ijI1}$. Constraints \eqref{eq:tw_retrieval_forbidden} explicitly forbid retrieval actions outside this valid interval, preventing premature or tardy deliveries.

\begin{align}
    & \sum_{i \in \mathcal{I} \setminus \{I\}} \sum_{j \in \mathcal{J}_i} \sum_{t=r_n - \tau_{ijI1}}^{r_n + \rho_n - \tau_{ijI1}} \sum_{v \in \mathcal{V}} y_{ijntv} = 1 \quad \forall n \in \mathcal{N} \label{eq:tw_retrieval} \\
    & \sum_{i \in \mathcal{I} \setminus \{I\}} \sum_{j \in \mathcal{J}_i} \left( \sum_{t=1}^{r_n - \tau_{ijI1} - 1} \sum_{v \in \mathcal{V}} y_{ijntv} + \sum_{t=r_n + \rho_n - \tau_{ijI1} + 1}^{T} \sum_{v \in \mathcal{V}} y_{ijntv} \right) = 0 \quad \forall n \in \mathcal{N} \label{eq:tw_retrieval_forbidden}
\end{align}
\noindent {\small \textit{Note:} Constraints (\ref{eq:tw_retrieval} and \ref{eq:tw_retrieval_forbidden}) specifically govern retrieval windows. For brevity, we omit the explicit storage constraints, as they are mathematically symmetric to the retrieval case: they restrict the storage actions $z_{ijntv}$ at the Source to start in the arrival window $[a_n, a_n+\alpha_n]$. (See Appendix \ref{sec:appendix_model})}
\medskip

Constraint \eqref{eq:tw_arrival} governs the entry logic for incoming unit loads, enforcing that each load $n$ is either stored in the buffer or directly cross-docked. The first term handles standard storage, ensuring the action $z$ occurs during the arrival window $[a_n, a_n+\alpha_n]$. The second term enables direct retrieval from the Source, which must satisfy two simultaneous conditions: the load must be available ($t \le a_n + \alpha_n$) and the resulting delivery to the Sink must meet the retrieval deadline ($t \ge r_n - \tau_{01I1}$). The upper bound of the summation enforces the tighter of these two limiting factors.

\begin{align}
    \sum_{i \in \mathcal{I}'} \sum_{j \in \mathcal{J}_i} \sum_{t=a_n}^{a_n+\alpha_n} \sum_{v \in \mathcal{V}} z_{ijntv} + \sum_{t=r_n - \tau_{01I1}}^{\min(r_n+\rho_n-\tau_{01I1}, a_n+\alpha_n)} \sum_{v \in \mathcal{V}} y_{01ntv} = 1 \quad \forall n \in \mathcal{N} \label{eq:tw_arrival}
\end{align}

\paragraph{Traffic Control}
We implement traffic rules by modeling each static lane $i$ as a unary resource that can accommodate at most one vehicle at any given time $t$. To formally capture lane occupancy during multi-step transitions, we define the set of relevant start times $\Omega(t, \tau) = \{ t' \in \mathcal{T} \mid t - \tau < t' \le t \}$. This set identifies all past time steps $t'$ where an action of duration $\tau$ initiated at $t'$ would still be active at the current time $t$.

Constraints \eqref{eq:monopoly} enforce the unary capacity by aggregating three distinct modes of occupation.
The first term accounts for \textit{static presence}, where an AMR is waiting in the lane ($c_{ijtv}$).
The second term captures \textit{incoming actions} (storage $z$, incoming reshuffling/empty travel $x, e$). Since the lane is blocked from the moment an AMR enters, we sum over the full duration window $\Omega(t, \tau)$.
The third term handles \textit{outgoing actions} (retrieval $y$, outgoing reshuffling/empty travel $x, e$). To avoid double-counting the vehicle's presence at the instant of departure (which is already captured by $c_{ijtv}$), we sum over the strictly past interval $\Omega^*(t, \tau) = \Omega(t, \tau) \setminus \{t\}$.
Collectively, the sum of these binary indicators must not exceed 1, ensuring collision-free operations.

\begin{align}
    & \sum_{v \in \mathcal{V}} \Bigg[ \sum_{j \in \mathcal{J}_i} c_{ijtv} 
    + \sum_{j \in \mathcal{J}_i} \biggl( \sum_{n \in \mathcal{N}} \sum_{t' \in \Omega(t, \tau_{lane}(j)+h)} z_{ijnt'v} + \sum_{k \in \mathcal{I}'} \sum_{l \in \mathcal{J}_k} \sum_{n \in \mathcal{N}} \sum_{t' \in \Omega(t, \tau_{lane}(j)+h)} x_{klijnt'v} \nonumber \\
    & \quad + \sum_{k \in \mathcal{I}} \sum_{l \in \mathcal{J}_k} \sum_{t' \in \Omega(t, \tau_{lane}(j))} e_{klijt'v} \biggr) \nonumber \\
    & \quad + \sum_{j \in \mathcal{J}_i} \biggl( \sum_{n \in \mathcal{N}} \sum_{t' \in \Omega^*(t, \tau_{lane}(j)+h)} y_{ijnt'v} + \sum_{k \in \mathcal{I}'} \sum_{l \in \mathcal{J}_k} \sum_{n \in \mathcal{N}} \sum_{t' \in \Omega^*(t, \tau_{lane}(j)+h)} x_{ijklnt'v} \nonumber \\
    & \quad + \sum_{k \in \mathcal{I}} \sum_{l \in \mathcal{J}_k} \sum_{t' \in \Omega^*(t, \tau_{lane}(j))} e_{ijklt'v} \biggr) \Bigg] \le 1 \quad \forall i \in \mathcal{I}', \forall t \in \mathcal{T} \label{eq:monopoly}
\end{align}

Finally, Constraints \eqref{eq:lifo} enforce the physical accessibility of the block storage. A slot $[i,j]$ is accessible only if the blocking slot $[i, j+1]$ is empty. Consequently, any action (reshuffling $x$, retrieval $y$, or empty travel $e$) targeting or originating from $[i,j]$ is forbidden if $b_{i(j+1)nt} = 1$.
\begin{align}
    & \sum_{n \in \mathcal{N}} (x_{ijklntv} + y_{ijntv} + e_{ijkltv}) \leq 1 - \sum_{n \in \mathcal{N}} b_{i(j+1)nt} \nonumber \\
    & \quad \forall i \in \mathcal{I}', \forall j \in \mathcal{J}_i \setminus \{J_i\}, \forall k \in \mathcal{I}', \forall l \in \mathcal{J}_k, \forall t \in \mathcal{T}, \forall v \in \mathcal{V} \label{eq:lifo}
\end{align}

\subsection{Computational Complexity}
\label{sec:complexity}
The BSRRP generalizes the Block Relocation Problem (BRP), which is known to be NP-hard \citep{CASERTA201296}. Consequently, the BSRRP is also NP-hard. This relationship can be established via a polynomial-time reduction where a standard BRP instance is mapped to a restricted BSRRP instance by fixing the fleet size to one ($|\mathcal{V}|=1$), eliminating arrivals ($a_n = 1$), setting empty travel costs to zero, and normalizing all loaded travel distances to unity ($d_{ijkl}=1$). Under these conditions, the BSRRP objective of minimizing total travel distance becomes mathematically equivalent to the BRP objective of minimizing the total number of relocations.

For a theoretical foundation, we provide a formal proof of this reduction in Appendix \ref{sec:nphard}. The proof explicitly constructs the mapping from BRP to BSRRP, demonstrating that intractability persists even under the specific constraints of our proposed model. Given this complexity, the EF is computationally intractable for large-scale instances, necessitating the heuristic approach described in Section \ref{sec:heuristic}.

%% file: sections/heuristic.tex
\section{Heuristic Approach}
\label{sec:heuristic}

The Exact Formulation proposed in Section \ref{sec:ip_model} provides an exact solution but becomes computationally intractable for large-scale instances due to the coupled complexity of multi-agent pathfinding, task scheduling, and the storing, retrieving and reshuffling of unit loads. To address this, we propose a hierarchical heuristic approach that decomposes the global optimization problem into four sequential, tractable sub-problems. This approach builds upon the multi-bay sorting strategies proposed by \citep{bomer2024sorting}, extending them to handle the storage and retrieval operations and the multi-AMR constraints of the BSRRP problem.

\input{figures/heuristic}

As illustrated in the flowchart, the proposed methodology processes the orders through four sequential stages. First, \textit{Priority Queue Generation} linearizes the asynchronous orders to provide a strict execution sequence for the \textit{Operation Sequencing via A* search}, to find a sequence of transportation tasks with storage, retrieval and reshuffling moves. Finally, \textit{Multi-AMR Scheduling} maps these moves to the AMR
fleet via Constraint Programming, while the subsequent \textit{Trajectory Repair} resolves any remaining spatiotemporal conflicts of the AMR positions using a three-tier priority mechanism to ensure collision-free execution.

Table \ref{tab:heuristic_notation} summarizes the notation used throughout the heuristic formulation. In the following each of the four components will be explained. 

\begin{table}[htbp]
\caption{Notation and parameters used in the heuristic approach}
\label{tab:heuristic_notation}
\centering
\begin{tabularx}{\columnwidth}{@{} l X @{}} 
\toprule
\textbf{Symbol} & \textbf{Description} \\
\midrule
\multicolumn{2}{l}{\textit{Sets and Indices}} \\
\midrule[0.15pt] 
\multicolumn{2}{c}{}\\[-1em] 
$\mathcal{O}$ & Set of storage and retrieval orders ($\mathcal{O} = \mathcal{O}_S \cup \mathcal{O}_R$) \\
$\mathcal{O}_S, \mathcal{O}_R$ & Subsets of storage and retrieval orders \\
$o, p$ & Indices for specific orders ($o, p \in \mathcal{O}$) \\
$\mathcal{G}$ & Set of dependency groups formed during priority assignment \\
$\mathcal{M}$ & Set of moves generated by A* search \\
$m_a, m_b$ & Indices for specific moves (storage, retrieval, reshuffling) \\
$\mathcal{V}$ & Set of Autonomous Mobile Robots (AMRs) \\
$u$ & Index for a specific unit load \\
\midrule
\multicolumn{2}{l}{\textit{Parameters and Variables}} \\
\midrule[0.15pt]
$[a_o, d_o]$ & Arrival and retrieval time window for order $o$ \\
$D_g$ & Effective deadline of a group ($\min_{o \in g} d_o$) \\
$Q$ & Final prioritized execution queue \\
$P(u)$ & Assigned priority value for unit load $u$ (lower value indicates higher urgency) \\
$o \sim p$ & Enclosure relationship ($o$ is nested within $p$) \\
$\theta$ & Tabu tenure for cleared lanes \\
$\sigma$ & Search state tuple $(B, \mathcal{U}_{\text{src}}, v_{\text{pos}})$ in A* search \\
$\Sigma'$ & Set of candidate successor states generated during expansion \\
$\gamma$ & Congestion scaling factor (set to 5) \\
$W$ & Penalty weight for time window deviations (set to 10,000) \\
$t_{\text{handling}}$ & Time required for load/unload operations \\ 
\midrule
\multicolumn{2}{l}{\textit{Functions and Variables}} \\
\midrule[0.15pt]
$g_{\text{time}}(\sigma)$ & Elapsed schedule time (makespan) at state $\sigma$ \\
$h_{\text{est}}(\sigma)$ & Total heuristic estimate cost \\
$h_{ops}$ & Operational Cost component \\
$h_{block}$ & Blocking Cost component \\
$h_{prio}$ & Priority Penalty \\
$h_{\text{pre\_store}}$ & Premature Storage Penalty \\
$\mathcal{A}_{m,v}$ & Interval variable for move $m$ assigned to vehicle $v$ \\
$\text{start}_{m,v}, \text{end}_{m,v}$ & Start and end times of move assignments \\ 
$\tau_{i,j}$ & Travel time between location $i$ and $j$ \\
\bottomrule
\end{tabularx}
\end{table}

\subsection{Priority Queue Generation}
\label{subsec:prio_queue}

The first stage of the heuristic transforms the set of asynchronous storage and retrieval requests into a linear execution sequence. Let $\mathcal{O}$ denote the complete set of orders, comprising both storage requests $\mathcal{O}_{S}$ and retrieval requests $\mathcal{O}_{R}$. Each order $o \in \mathcal{O}$ is constrained by a time window $[a_o, d_o]$, where $a_o$ represents the earliest release time and $d_o$ the hard deadline.

Standard sorting strategies, such as Earliest Due Date (EDD), often fail in scenarios with asynchronous release times, where a task with a late deadline might have a very late release time (a tight window), rendering it more critical than a task with an earlier deadline but a wide window. Executing the task with the wider window first may irreversibly consume temporal resources required by the tighter task immediately upon its release.

\begin{algorithm}
\caption{Enhanced Earliest Due Date Strategy}
\label{alg:prio_assignment}
\DontPrintSemicolon
\KwIn{Set of orders $\mathcal{O}$}
\KwOut{Priority Queue $Q$}

$\mathcal{G} \leftarrow$ Partition $\mathcal{O}$ into groups using enclosure relation $o \sim p$ (Eq. 1)\;
Sort $\mathcal{G}$ ascending by group deadline $D_g = \min_{o \in g} (d_o)$\;
$Q \leftarrow$ Concatenate orders from sorted groups, locally sorted by $d_o$\;
\Return $Q$
\end{algorithm}

To resolve this, we employ an Enhanced Earliest Due Date strategy, summarized in Algorithm \ref{alg:prio_assignment}, which is based on the concept of enclosed windows. We define an enclosure relationship $o \sim p$ between two orders $o, p \in \mathcal{O}$ if the time window of one is entirely contained within the other:
\begin{equation}
    o \sim p \iff (a_o \leq a_p \land d_o \geq d_p) \lor (a_p \leq a_o \land d_p \geq d_o)
\end{equation}
This binary relationship identifies pairs of tasks that compete for the same time interval. We utilize this relationship to partition the set of orders $\mathcal{O}$ into disjoint groups, effectively treating the orders as nodes in a graph where an edge exists between any pair $o, p$ if $o \sim p$. The resulting connected components form the groups, clustering temporally dependent tasks.

The heuristic then generates the final priority queue by sorting these groups to ensure resource availability for critical tasks:
\begin{enumerate}
    \item \textbf{Inter-Group Sorting:} The groups are sorted by the minimum deadline of their constituent orders ($\min_{o \in \text{group}} d_o$). This ensures that a cluster containing even a single urgent order is prioritized over a cluster of flexible tasks.
    \item \textbf{Intra-Group Sorting:} Within each group, orders are sorted by their individual deadlines $d_o$.
\end{enumerate}

This two-level sorting yields a strict priority ordering that explicitly protects orders with nested, tight windows from being preempted by non-critical tasks. Based on their final position in the queue $Q$, each unit load $u$ is assigned an integer priority value $P(u)$, where a lower numerical value indicates a higher urgency. For example, a sequence of four unit loads might receive the priority assignments $P(u_1)=1$, $P(u_3)=1$ (if they share the same critical deadline), $P(u_4)=2$, and $P(u_2)=3$. During the subsequent $A^{}$ search, these priority values are used to calculate penalties for invalid or non-optimal placement sequences.

\subsection{Operation Sequencing via A* Search}
\label{subsec:astar}

The second stage converts the prioritized sequence of orders $\mathcal{O}$ into a dependency graph of moves $\mathcal{M}$, which explicitly defines the locations for all storage, reshuffling, and retrieval operations. This process employs a two-phase approach: first, a pre-processing step prunes trivial direct transfers from the source to the sink to reduce the search space, followed by an A* search that enforces the priority scheme established in Section~\ref{subsec:prio_queue} and determines storage locations for the orders by using a virtual AMR as a substitute for the robot fleet.

\subsubsection{Direct Retrieval Pruning}
Before initializing the search, we perform a \textit{Direct Retrieval Analysis} to identify unit loads that can bypass the buffer entirely. A unit load $u$ is extracted for direct retrieval if its release time $a_u$ and deadline $d_u$ allow for a direct transfer from Source to Sink. Formally, if $a_u + \tau_{\text{source}, \text{sink}} \leq d_u$, the item is removed from the search space and assigned a direct move. This reduces the branching factor for the subsequent pathfinding.

\subsubsection{State Space and Transitions}
For the remaining orders, we search for an optimal sequence of operations. The state space is defined by the tuple $\sigma = (B, \mathcal{U}_{src}, v_{pos})$, where B represents the current buffer configuration (mapping unit loads to lanes and slots), $\mathcal{U}_{src}$ is the set of pending unit loads at the source, and $v_{pos}$ tracks the position of a single virtual AMR to estimate and minimize the empty travel distances between consecutive transport tasks. Representing the fleet as a single virtual AMR is a necessary methodological simplification for fleet sizes $|\mathcal{V}| > 1$. Tracking the individual positions of multiple AMRs within the A* state space would lead to a combinatorial explosion, rendering the search intractable. By optimizing the sequence for a single virtual AMR, the heuristic inherently groups spatially proximate tasks and minimizes total empty travel. This produces a highly cohesive operation sequence that the subsequent CP-SAT scheduling stage can efficiently distribute and parallelize across the actual AMR fleet.
Transitions between states correspond to three move types:
\begin{itemize}
    \item Store: Places an arriving unit at the source into a valid empty slot in the buffer.
    \item Retrieve: Retrieves an accessible unit load from the front of a lane to the Sink.
    \item Reshuffle: Relocates a blocking unit load to a temporary position.
\end{itemize}

A state $\sigma$ is identified as a goal state ($IsGoal$) when all storage and retrieval orders from the priority queue $Q$ have been successfully executed and the buffer state is consistent. Once a goal is reached, the algorithm uses $ReconstructPath$ to backtrack through the state transitions, yielding the final move sequence $\mathcal{M}$ for the subsequent scheduling phase.

To ensure scalability while maintaining the solution quality of standard $A^{*}$, we employ three key algorithmic optimizations. First, to bound the branching factor in high-density scenarios, we implement a Beam Search strategy \citep{lowerre1976harpy}. At each expansion step, the generated successors are ranked by their $f$-cost, and only the top $k$ most promising candidates (with $k=8$ in our experiments) are added to the open set. Second, we enforce Open Set Pruning as a memory protection mechanism. If the priority queue exceeds a safety threshold (set to 5,000 nodes), the worst-performing 50\% of the nodes are discarded to prevent memory exhaustion and search stagnation. Third, we utilize a lazy evaluation strategy for the heuristic cost. Upon node generation, only the computationally inexpensive operational costs are calculated. The expensive penalty components (blocking and priority violations) are deferred and computed only when the node is extracted from the priority queue, preventing wasted computation on unexplored nodes. The complete search procedure, integrating the beam search and memory protection strategies, is summarized in Algorithm \ref{alg:astar}.

\begin{algorithm}
\caption{A* search}
\label{alg:astar}
\DontPrintSemicolon
\KwIn{Initial State $\sigma_{init}$, Priority Queue $Q$}
\KwOut{Move Sequence $\mathcal{M}$}

$OPEN \leftarrow \{ (\sigma_{init}, 0) \}$\;

\While{$OPEN \neq \emptyset$}{
    Select state $\sigma$ with lowest $f(\sigma)$ from $OPEN$\;
    \If{IsGoal($\sigma$)}{ \Return ReconstructPath($\sigma$) }

    \vspace{0.1cm}
    \textbf{Step 1: Expansion \& Tabu Validation}\;
    \Indp
    Generate successors $\Sigma'$ by applying actions $\{Store, Retrieve, Reshuffle\}$\;
    Filter $\Sigma'$: Remove actions violating Tabu tenure $\theta$ (Cycle Prevention)\;
    \Indm

    \vspace{0.1cm}
    \textbf{Step 2: Evaluation \& Beam Pruning}\;
    \Indp
    Calculate cost $f(\sigma') = g_{\text{time}} + h_{\text{est}}$ for all $\sigma' \in \Sigma'$\;
    Sort $\Sigma'$ by $f$-score and add only top $k$ nodes to $OPEN$ (Beam Width)\;
    \Indm

    \vspace{0.1cm}
    \textbf{Step 3: Memory Protection}\;
    \Indp
    \lIf{$|OPEN| > Limit$ }{Discard worst 50\% of nodes from $OPEN$\; }
    \Indm
} 
\Return Failure
\end{algorithm}

To prevent cyclic behavior (e.g., immediate refilling of a cleared lane), the search maintains a short-term Tabu list. When a unit load is retrieved or reshuffled from a lane $l$, that lane becomes tabu for incoming storage or reshuffling moves.
To balance fleet coordination with unrestricted lane access, we set a fixed short-term tenure of $\theta = 1$ state transition. This tenure effectively prevents immediate inverse operations (such as placing a unit load back into the position it just vacated) without restricting the solution space or locking down buffer capacity for extended periods, which proved more robust than dynamic, fleet-dependent tenures in testing.

\subsubsection{Cost Function}
The search is guided by a composite cost function $f(\sigma) = g_{\text{time}}(\sigma) + h_{\text{est}}(\sigma)$.

\paragraph{Total Schedule Makespan} The term $g_{\text{time}}(\sigma)$ represents the total elapsed time of the operation sequence up to state $\sigma$. Unlike standard pathfinding which might only sum loaded distances, our cost function accounts for the empty travel time required to travel between task locations, as well as any waiting time incurred if a vehicle arrives before a unit load's release window $a_o$ opens. This is achieved by employing a virtual AMR that sequentially performs the tasks, calculating the unloaded travel from the end position of the preceding move to the start location of the next. This allows penalizing inefficient sequencing (e.g., generating excessive empty travel for the virtual AMR between distant tasks) and optimize for the true operational makespan.

\paragraph{Heuristic Cost Function}
The heuristic $h_{\text{est}}(\sigma)$ is a weighted sum of four components designed to minimize future effort while enforcing the priority ordering. To maintain search tractability and ensure robust performance in high-density scenarios, the heuristic employs soft constraints within this cost function. Rather than strictly forbidding undesirable states—which could lead to search stagnation—it penalizes violations (such as lane blockages or sequence inversions) through weighted penalty terms.

\begin{equation}
h_{\text{est}}(\sigma) = h_{\text{ops}}(\sigma) + h_{\text{block}}(\sigma) + h_{\text{prio}}(\sigma) + h_{\text{pre\_store}}(\sigma) \end{equation}

\begin{enumerate}
    \item \textbf{Operational Costs ($h_{\text{ops}}$):} Estimates the travel time $\tau$ required to complete all pending tasks.
    \begin{itemize}
        \item \textit{Storage:} We employ a \textbf{Greedy Best-Match} heuristic. The cost for storing pending unit loads is estimated by assigning each item to the nearest available slot that minimizes the total sequence: Source $\rightarrow$ Slot $\rightarrow$ Sink.
        \item \textit{Retrieval:} Sums the travel time $\tau_{i, \text{sink}}$ from each stored item's current position to the Sink.
    \end{itemize}

    \item \textbf{Blocking Costs ($h_{\text{block}}$):} Anticipates the immediate reshuffling effort required for currently blocked targets. For every blocked target $u$, the heuristic identifies the set of blockers and calculates the cost to relocate them to the nearest empty lane. If no empty lane is available, the algorithm applies two times the average reshuffle cost to reflect the operational delay of waiting for future retrievals to free up buffer capacity. This cost is scaled by a fixed penalty multiplier $\gamma = 5.0$. We found this fixed value to be effective for consistent congestion avoidance, as it ensures high-priority reshuffling is penalized uniformly regardless of the fleet size.

    \item \textbf{Priority Penalty ($h_{\text{prio}}$):} Enforces the task sequence derived in Stage 1 by penalizing two types of structural violations based on the assigned priority values $P(u)$:
    \begin{itemize}
        \item \textit{Sequence Inversion:} Penalizes the retrieval of a lower-priority unit load while a higher-priority one is still pending in the buffer. \\
        \textit{Example:} Retrieving an item with $P=3$ while an item with $P=1$ is not retrieved yet.
        \item \textit{Stacking Violation:} Penalizes any LIFO lane configuration where a lower-priority item is placed closer to the aisle than a higher-priority item. \\
        \textit{Example:} Placing an item with $P=2$ in front of an item with $P=1$ in the same lane. Since this placement guarantees a future reshuffle operation, it is penalized immediately to prune the search branch.
    \end{itemize}

    \item \textbf{Premature Storage Penalty ($h_{\text{pre\_store}}$):} Penalizes the storage of a low-priority unit load from the source while a high-priority unit load is waiting to be retrieved from the buffer. This guides the search to clear urgent retrievals first, freeing up buffer space before bringing in less urgent items.
\end{enumerate}

The output of this stage is a sequence of moves $\mathcal{M}$. If move $m_a$ reshuffles a blocker for a retrieval move $m_b$, a strict precedence constraint $m_a \prec m_b$ is generated for the subsequent scheduling phase.

\subsection{Multi-AMR Scheduling via CP-SAT}
\label{subsec:scheduling}

The third stage maps the sequence of moves $\mathcal{M}$ generated by the A* search onto the fleet of Autonomous Mobile Robots (AMRs) $\mathcal{V}$. While standard Vehicle Routing Problems (VRP) focus primarily on spatial routing, the BSRRP is dominated by complex temporal constraints and precedences between moves. Consequently, we model this stage as a \textit{Job Shop Scheduling Problem with Sequence-Dependent Setup Times}, solved using the CP-SAT solver from Google OR-Tools.

\subsubsection{Model Formulation}
To emphasize the scheduling nature of the problem, the vehicles of the AMR fleet are modeled as a set of identical parallel machines $\mathcal{V} = \{1, \dots, V\}$. Each generated move $m \in \mathcal{M}$ is treated as a task that must be assigned to exactly one machine. Following standard Constraint Programming formulation, we define optional interval variables $\mathcal{A}_{m,v}$ to represent the potential execution of move $m$ by vehicle $v$. If active, an interval $\mathcal{A}_{m,v}$ is characterized by a start time $\text{start}_{m,v}$, an end time $\text{end}_{m,v}$, and a fixed processing duration $\tau_m$. This duration combines the loaded travel time (as defined in Section \ref{sec:ip_model}) and the constant handling time:
\begin{equation}
    \tau_m = \tau_{\text{origin}(m), \text{dest}(m)} + t_{\text{handling}} 
\end{equation}

The model enforces the following constraints:

\begin{enumerate} 
    \item \textbf{Assignment Completeness:} Every move must be assigned to exactly one vehicle. This is enforced by constraining the sum of active assignment booleans to 1 for each move: \begin{equation} \sum_{v \in \mathcal{V}} \mathbb{I}(\mathcal{A}_{m,v}) = 1 \quad \forall m \in \mathcal{M} \end{equation}

    \item \textbf{Precedence Constraints:} We enforce three types of hard dependencies to guarantee consistency and stack integrity:
    \begin{itemize}
        \item \textit{Unit Load Flow:} Sequential operations acting on the same unit load (e.g., reshuffle blocker $\rightarrow$ retrieve target) must maintain the order defined by the logical flow.
        \item \textit{LIFO Inter-Slot Dependencies:} Derived from the buffer geometry; if unit load $A$ is stored in a slot strictly in front of unit load $B$ within the same lane, the retrieval of $A$ must complete before the retrieval of $B$ can commence.
        \item \textit{Lane Sequencing:} To preserve the validity of the buffer states determined by the heuristic decomposition, we enforce strict sequencing for all moves accessing the same lane. If the A* search schedules move $m_a$ before $m_b$ on lane $l$, the scheduler is constrained to respect this order ($m_a \prec m_b$), preventing the optimizer from creating invalid lane configurations.
    \end{itemize}
    For any such dependency pair $(m_a, m_b)$, we enforce $\text{end}_{m_a} \leq \text{start}_{m_b}$.

    \item \textbf{Lane Capacity :} We model storage locations as resource constraints.
    \begin{itemize}
        \item \textit{Buffer Lanes (Unary Resource):} Standard buffer lanes are modeled as unary resources. We apply a global \textit{NoOverlap} constraint on the set of intervals assigned to any specific lane $l$:
        \begin{equation}
            \text{NoOverlap}(\{ \mathcal{A}_{m,v} \mid \text{dest}(m) = l \lor \text{origin}(m) = l \})
        \end{equation}
        This constraint ensures that at any point in time $t$, at most one vehicle can execute a move involving lane $l$.
        
        \item \textit{Source and Sink Queues:} Modeled as infinite-capacity resources to track vehicle availability without restricting the number of concurrent AMRs at these locations.
    \end{itemize}

    \item \textbf{Sequence-Dependent Transition Times:} The setup time between two moves depends on the robot's location. If vehicle $v$ performs move $m_a$ immediately before $m_b$, a transition constraint ensures the gap covers the empty travel time:
    \begin{equation}
        \text{start}_{m_b, v} \geq \text{end}_{m_a, v} + \tau_{\text{dest}(m_a), \text{origin}(m_b)}
    \end{equation}

    \item \textbf{Time Windows:} Time windows are modeled using a hybrid approach to ensure feasibility.
    \begin{itemize}
        \item \textit{Hard Constraints:} Physical constraints are strictly enforced (e.g., a storage action cannot start before the unit load's arrival time $a_o$; a retrieval cannot end after the deadline $d_o$).
        \item \textit{Soft Constraints:} Operational targets (latest start, earliest finish) are treated as soft constraints. Violations are permitted to maintain feasibility but incur a weighted tardiness penalty in the objective function to prioritize service-level agreements.
    \end{itemize}
\end{enumerate}

\subsubsection{Objective Function}
The optimization objective is hierarchical, implemented using a weighted sum method. The primary goal is to minimize total tardiness, reflecting the strict service level requirements. The secondary goal is to minimize the sum of completion times (Flow Time), which implicitly minimizes unproductive empty travel and waiting times.

The objective is formulated as:
\begin{equation} 
    \min \sum_{m \in \mathcal{M}} \left( \underbrace{W \cdot \max(0, \text{end}_m - d_m)}_{\text{Weighted Tardiness}} + \underbrace{\text{end}_m}_{\text{Flow Time}} \right) 
\end{equation}

Here, the term $\max(0, \text{end}_m - d_m)$ calculates the strictly positive delay relative to the soft deadline $d_m$. $W$ is a large weighting constant (set to 10,000) ensuring that meeting deadlines strictly dominates operational efficiency.

\subsection{Trajectory Repair} 
\label{subsec:repair}

\begin{algorithm}
\caption{Trajectory Repair Strategy}
\label{alg:repair}
\DontPrintSemicolon
\KwIn{Initial Schedule $\mathcal{S}$}
\KwOut{Feasible Schedule $\mathcal{S}^*$}

\vspace{0.1cm}
\textbf{Step 1: Deadlock Resolution}\;
\lIf{Symmetric deadlock detected between $v_1, v_2$}{Swap schedules of $v_1, v_2$}

\vspace{0.2cm}
\textbf{Step 2: Conflict Resolution Loop}\;
\While{Collision detected between $v_1$ (parking AMR) and $v_2$ (incoming AMR)}{
    \uIf(\tcp*[f]{Priority 1: Reschedule}){$v_1$ is waiting empty \textbf{and} early shift feasible}{
        Shift $v_1$ departure to $t < t_{arrival}(v_2)$\;
    }
    \uElseIf(\tcp*[f]{Priority 2: Evict}){valid eviction strategy for $v_1$ exists}{
        Insert best eviction move (Smart or Standard) for $v_1$\;
    }
    \Else(\tcp*[f]{Priority 3: Delay}){
        Delay $v_2$ until $v_1$ departs (propagate downstream)\;
    }
}

\vspace{0.2cm}
\textbf{Step 3: Dependency Check}\;
\ForEach{unit load $u$ with $\text{end}_{store} > \text{start}_{retrieve}$}{
    Shift retrieval forward to resolve violation\;
}

\Return $\mathcal{S}^*$
\end{algorithm}

While the CP-SAT schedule ensures temporal validity, it ignores the presence of idle AMRs, which may park after performing a storage or reshuffling move and occupy the buffer lanes. To generate collision-free trajectories, the final stage post-processes the timeline following the structure of Algorithm \ref{alg:repair}:

\paragraph{Step 1: Deadlock Resolution}
The algorithm resolves symmetric head-on collisions (e.g., $v_1$ moving $A \rightarrow B$ while $v_2$ moves $B \rightarrow A$) by swapping the AMRs' future task schedules. Since the AMRs are homogenous, this resolves the deadlock instantly without additional travel time.

\paragraph{Step 2: Conflict Resolution Loop}
Remaining spatial overlaps caused by parked AMRs are resolved using a hierarchical strategy prioritizing efficiency:
\begin{enumerate}
    \item \textbf{Priority 1 Reschedule:} Exploits schedule slack to shift the parked AMR's next departure to an earlier time, clearing the lane before the incoming AMR arrives.
    \item \textbf{Priority 2 Evict:} Inserts an explicit eviction move. The algorithm prioritizes a \textit{Smart Eviction} (moving the AMR directly to the start location of its next assigned task) over a \textit{Standard Eviction} (relocating it to an available neutral position, such as a free buffer lane or the sink).
    \item \textbf{Priority 3 Delay:} As a fallback, the incoming AMR is delayed, and this shift is propagated downstream to all dependent tasks.
\end{enumerate}

\paragraph{Step 3: Dependency Check}
Delays introduced in Step 2 may violate precedence constraints (e.g., pushing a storage task to complete after its retrieval was scheduled to begin). This step enforces the $\text{end}_{\text{store}} \leq \text{start}_{\text{retrieve}}$ constraint for all unit loads, shifting retrieval tasks forward if necessary to guarantee consistency.

%% file: figures/heuristic.tex
\begin{figure}[htbp]
\centering
\resizebox{\textwidth}{!}{%
\begin{tikzpicture}[
    process/.style={
        rectangle, draw=black, thick, fill=white,
        text width=3.4cm, minimum height=1.4cm, 
        align=center, rounded corners, font=\small
    },
    io/.style={
        rectangle, draw=black, dashed, fill=white,
        text width=2.4cm, minimum height=1.0cm, 
        align=center, font=\small
    },
    arrow/.style={-Stealth, thick},
    connector/.style={-Stealth, thick, rounded corners=5pt},
    node distance=0.6cm
]

\node (input) [io, solid, thick] {\textbf{Input}\\ Orders $\mathcal{O}$ \\ Time Windows};

\node (step1) [process, right=of input] {\textbf{Stage 1: Priority Queue Generation}};

\node (data1) [io, right=of step1] {Linearized Task Sequence};

\node (step2) [process, right=of data1] {\textbf{Stage 2: Operation Sequencing via $A^*$ Search}};

\node (data2) [io, right=of step2] {Transportation Tasks Sequence};

\path (input.center) -- ++(0,-2.2) coordinate (row2_y);

\node (step3) [process] at (step1 |- row2_y) {\textbf{Stage 3: Multi-AMR Scheduling}};

\node (data3) [io] at (data1 |- row2_y) {Schedules for the AMRs};

\node (step4) [process] at (step2 |- row2_y) {\textbf{Stage 4: Trajectory Repair}};

\node (output) [io, solid, thick] at (data2 |- row2_y) {\textbf{Output}\\ Feasible Multi-AMR Schedules};

\draw [arrow] (input) -- (step1);
\draw [arrow] (step1) -- (data1);
\draw [arrow] (data1) -- (step2);
\draw [arrow] (step2) -- (data2);

\draw [arrow] (step3) -- (data3);
\draw [arrow] (data3) -- (step4);
\draw [arrow] (step4) -- (output);

\draw [connector] (data2.south) -- ++(0,-0.6) -| (step3.north);

\end{tikzpicture}
}
\caption{Overview of the hierarchical heuristic. The sequential stages transform the input from a linearized task sequence into strict precedence constraints, then into timed assignments, and finally into collision-free trajectories.}
\label{fig:heuristic_overview}
\end{figure}
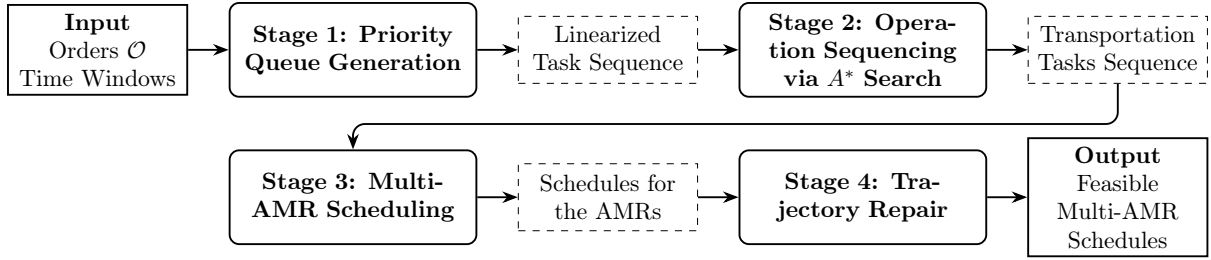

%% file: sections/experiments.tex
\section{Computational Experiments}
\label{sec:computational_experiments}

This section evaluates the performance of the proposed solution approaches for the Multi-AMR Buffer Storage, Retrieval, and Reshuffling Problem (BSRRP). The experimental study is designed to answer three primary research questions:

\begin{enumerate}
    \item \textbf{Benchmarking:} What are the computational limits of the EF when establishing a ground truth of optimal solutions?
    \item \textbf{Heuristic Validation:} How does the proposed hierarchical heuristic perform in terms of feasibility rates and solution quality compared to the optimal baselines?
    \item \textbf{Managerial Insights:} How do operational parameters—specifically fleet size, access flexibility, and congestion levels—impact the stability and efficiency of space-constrained buffers?
\end{enumerate}

\subsection{Experimental Design and Dataset}
\label{subsec:experimental_design}

To ensure a robust evaluation, we developed a comprehensive dataset reflecting the constraints of brownfield manufacturing facilities, such as limited floor space, high storage density, and complex traffic dynamics. To guarantee full reproducibility and facilitate future research, the complete source code of the proposed heuristic and the discrete-event simulator, as well as all generated benchmark instances and detailed solutions, are made publicly available (see the Data and Code Availability Statement in the Acknowledgements). Our experimental design employs a two-tiered validation to address different evaluation objectives: solution quality benchmarking on small-scale instances and scalability analysis on large-scale instances.

\subsubsection{Instance Generation and Parameters}

\paragraph{Small-Scale Instances}
For $3 \times 3$ and $4 \times 4$ grids, instances were generated stochastically to quantify the heuristic's optimality gap against the EF. We systematically varied the following parameters:
\begin{itemize}
    \item \textbf{Grid and Topology:} $3 \times 3$ (9 slots) and $4 \times 4$ (16 slots) with access points distributed across 1, 2, or 4 sides.
    \item \textbf{Fleet Size ($|\mathcal{V}|$):} 1 to 3 AMRs.
    \item \textbf{Load-to-Slot Ratio:} Ranged from 0.4 to 1.3 to test capacity limits.
    \item \textbf{Temporal Constraints:} Arrival and retrieval windows were generated with varying overlap (using random seeds for reproducibility) to enforce asynchronous operations. Negative arrival time windows were utilized to initialize instances with a set of already randomly placed pre-stored unit loads at $t=0$.
\end{itemize}

\paragraph{Large-Scale Instances} 
For layouts exceeding exact computational limits, we developed a discrete-event constructive simulator. This tool generates task sequences by simulating buffer operations forward in time. It guarantees feasibility by explicitly simulating the necessary reshuffling of blocking unit loads, while intentionally ignoring AMR collisions, creating an idealized, continuous flow of operations. We utilize these densely packed sequences to stress-test the heuristic, allowing us to identify the saturation points. Specifically, the maximum unit load-to-slot-ratio a layout can sustain before the heuristic fails to find a valid schedule within the simulated deadlines.
\begin{itemize}
    \item \textbf{Topology:} Square blocks ($5 \times 5$, $6 \times 6$), rectangular layouts ($8 \times 3$), and industrial brownfield layouts.
    \item \textbf{Task Generation:} The simulator adaptively alternates storage and retrieval requests to maintain a target fill level (e.g., 80\%), utilizing a back-to-front filling strategy.
    \item \textbf{Idealized Time Estimation:} Task durations are estimated using scaled Manhattan distances ($t_{op} \approx 2.0 \times d_{\text{manhattan}}$) plus fixed handling times, and reshuffling operations incur explicitly simulated time penalties.
    \item \textbf{Time Window Derivation:} Task start times are greedily assigned to the earliest available robot in a simulated fleet with two AMRs. The arrival ($[a_n, a_n+\alpha_n]$) and retrieval time windows ($[r_n, r_n+\rho_n]$) are then generated by applying a fixed temporal slack (e.g., $\pm 15$ time steps) around these simulated task start and end times.
\end{itemize}

\subsubsection{Computational Environment}
All experiments were conducted on a workstation equipped with an AMD Ryzen 9 5950X processor. The EF was solved using Gurobi Optimizer version 13.0.0 with a strict time limit of 3600 seconds (1 hour) per instance.

The hierarchical heuristic was implemented in Python, utilizing Google OR-Tools (CP-SAT) for the scheduling stage. To ensure rapid convergence and stability during the experiments, the solver was configured with aggressive search parameters (linearization level 2, probing level 2) and utilized 8 parallel search workers. To simulate a realistic production control environment, we imposed a realistic time limit of 300 seconds for both the $A^*$ Search (Stage 2) and the CP-SAT solver (Stage 3). While typical runtimes are fractions of a second, this cap prevents indefinite stalling in degenerate cases. Additionally, the algorithmic parameters of the heuristic (specifically the congestion factor $\gamma = 5$, penalty weight $W = 10,000$ and tabu tenure $\theta = 1$) were calibrated based on preliminary experiments using a representative subset of small-scale instances to balance solution quality and computational speed.

\subsubsection{Exact Formulation Experiments and Benchmark Set}

To evaluate the EF and establish a ground truth, we generated a pool of 6,903 small-scale instances ($3 \times 3$ and $4 \times 4$) across varying layout complexities and parameters. We subjected these instances to a two-stage process to test the computational limits of the EF and to build a benchmark set.

First, the EF was tasked with finding feasible solutions within a 3600-second time limit. This step identified 949 instances as solvable. We assume this step effectively separates operationally feasible scenarios from those rendered structurally impossible by tight temporal and spatial constraints. This assumption is based on the observation that when the EF found a feasible solution, it typically did so relatively quickly, which we verified by examining individual instances.

Second, to ensure a strict baseline for optimality gap calculations, we filtered this pool to instances where the EF achieved a solution with a MIP gap of $\le 5\%$. This process yielded a final benchmark set of 810 instances. The composition of this benchmark, categorized by fleet size and access topology, is detailed in Table 4. The distribution highlights the computational boundaries of exact methods: the EF solved 716 instances in the $3 \times 3$ layout to the required gap, but only 94 instances in the $4 \times 4$ layout. Notably, the EF failed to solve any $4 \times 4$ instances with a single AMR, suggesting that the used parameter combination exceeds the capacity of a single robot. 

We utilize this filtered benchmark set of 810 optimally or near-optimally solved instances ($\le 5\%$ MIP gap) during the subsequent evaluation to validate the solution quality and feasibility rates of the proposed heuristic.

\begin{table}[ht!]
\centering
\caption{Composition of the Benchmark Reference Set: 810 high-quality instances filtered from a pool of 6,903 generated small-scale scenarios.}
\label{tab:validation_summary}
\resizebox{0.7\textwidth}{!}{%
\begin{tabular}{llcc}
\toprule
\textbf{Parameter} & \textbf{Category} & \textbf{Small ($3 \times 3$)} & \textbf{Large ($4 \times 4$)} \\ 
\midrule
\multirow{3}{*}{Fleet Size ($|\mathcal{V}|$)} 
 & 1 AMR & 151 & 0 \\
 & 2 AMRs & 288 & 48 \\
 & 3 AMRs & 277 & 46 \\ 
\midrule
\multirow{3}{*}{Access Directions} 
 & 1 Side & 72 & 1 \\
 & 2 Sides & 373 & 54 \\
 & 4 Sides & 271 & 39 \\ 
\midrule
\textbf{Total} & \textbf{All Instances} & \textbf{716} & \textbf{94} \\ 
\bottomrule
\end{tabular}%
}
\end{table}

\subsection{Quantitative Benchmarking against Exact Formulation}
\label{subsec:heuristic_validation}

We first analyze the quantitative performance of the heuristic against the benchmark set, followed by a qualitative assessment of its conflict resolution capabilities.

\subsubsection{Quantitative Analysis}
\label{subsec:computational_efficiency}

Table \ref{tab:combined_performance} summarizes the comparative performance of the exact formulation (EF) and the proposed heuristic. The results are first categorized by the physical layout and the fleet size ($|\mathcal{V}|$). Under the section \textit{Solve Rate Heuristic}, the first column gives the total number of benchmark instances, the second column lists how many were successfully solved by the heuristic, and the third column provides the resulting success rate. In the \textit{Computational Efficiency (Median)} section, the table compares the median EF runtime against the median heuristic runtime, followed by the relative speedup factor. Finally, the \textit{Median Opt. Gap (\%)} column quantifies the solution quality by measuring the objective value deviation of the heuristic solution from the exact lower bound.

\begin{table}[ht!]
\centering
\caption{Performance assessment: Feasibility, Runtime, and Optimality Gap comparison between the EF and the Heuristic approach across the reference set.}
\label{tab:combined_performance}
\resizebox{\textwidth}{!}{%
\begin{tabular}{ll|ccc|ccc|c}
\toprule
 &  & \multicolumn{3}{c|}{\textbf{Solve Rate Heuristic}} & \multicolumn{3}{c|}{\textbf{Computational Efficiency (Median)}} & \textbf{Median Opt.} \\
\textbf{Layout} & \textbf{Fleet ($|\mathcal{V}|$)} & \textbf{Number of Instances} & \textbf{Solved} & \textbf{Rate} & \textbf{EF Runtime} & \textbf{Heur. Runtime} & \textbf{Speedup} & \textbf{Gap (\%)} \\ 
\midrule
\multirow{3}{*}{$3 \times 3$} 
 & 1 AMR & 151 & 139 & 92.1\% & 515.5 s & 0.05 s & 10,310x & 2.56\% \\
 & 2 AMRs & 288 & 284 & 98.6\% & 253.0 s & 0.08 s & 3,163x & 10.20\% \\
 & 3 AMRs & 277 & 275 & 99.3\% & 307.2 s & 0.09 s & 3,413x & 14.58\% \\ 
\midrule
\multirow{2}{*}{$4 \times 4$} 
 & 2 AMRs & 48 & 44 & 91.7\% & 1,524.3 s & 3.37 s & 452x & 4.70\% \\
 & 3 AMRs & 46 & 45 & 97.8\% & 2,927.9 s & 2.52 s & 1,162x & 10.63\% \\ 
\bottomrule
\end{tabular}%
}
\end{table}

\paragraph{Solvability and Structural Limits}
The proposed heuristic demonstrated high consistency in finding feasible solutions when evaluated against the benchmark. As detailed under the \textit{Solve Rate Heuristic} section of Table~\ref{tab:combined_performance}, the heuristic achieved success rates comparable to the EF across all parameter combinations where exact solutions were obtainable.

Specifically, the heuristic solved 698 out of 716 instances (97.5\%) in the $3\times3$ layout. In the more complex $4\times4$ layout, it successfully solved 89 out of the 94 reference instances (94.7\%). Notably, the EF was unable to solve any $4\times4$ instances restricted to a single AMR. As outlined in Section~\ref{subsec:experimental_design}, this is a direct consequence of the instance generation parameters: the extensive reshuffling required to resolve deep blockages in a $4\times4$ grid simply consumes more time than the tight retrieval windows allow for a single AMR, rendering these scenarios operationally infeasible.

\begin{figure}[ht!]
    \centering
    \includegraphics[width=\textwidth]{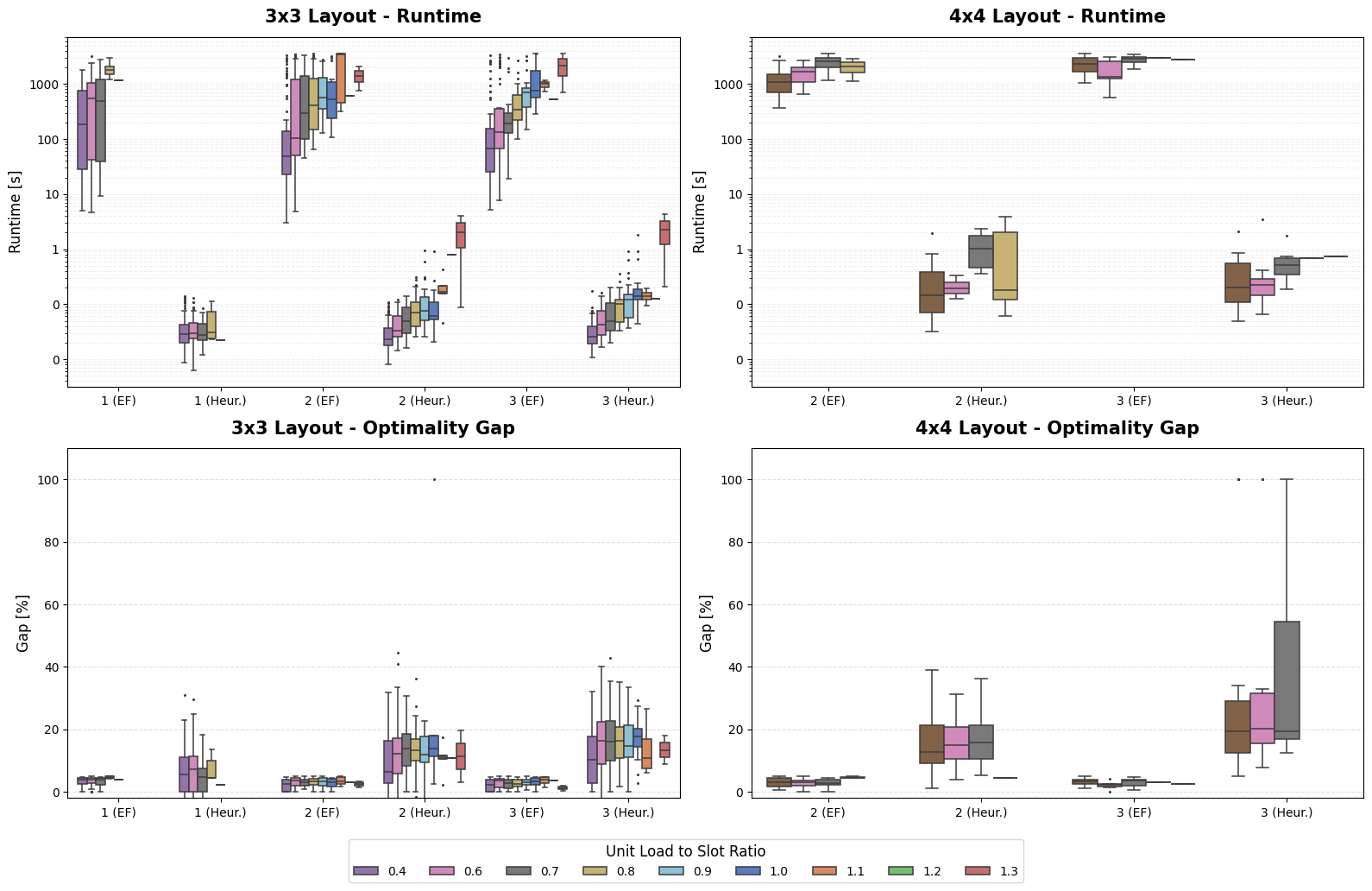}
    \caption{Performance comparison between the EF and the proposed Heuristic. Top panels: Runtime distributions (logarithmic scale) across $3\times3$ and $4\times4$ layouts. Bottom panels: Optimality gaps relative to the exact lower bound, categorized by unit-load-to-slot-ratio.}
    \label{fig:performance_analysis}
\end{figure}

\paragraph{Runtime Performance}
While feasibility is the primary prerequisite for deployment, operational viability depends on computational speed. The runtime disparity, visualized in the top panels of Figure~\ref{fig:performance_analysis}, highlights the impact of the problem's NP-hard complexity. While the EF frequently approaches the timeout of 3,600 seconds (e.g., median of 2,927s for $4\times4$ with 3 AMRs), the heuristic maintains stable sub-second to low-second runtimes. For the $3\times3$ instances, the median computation time was consistently under 0.1 seconds, achieving speedup factors exceeding 3,000x. Even in the complex $4\times4$ instances, the median was approximately 3 seconds—well below the 300-second time limit.

\paragraph{Solution Quality and Optimality Gap}
To assess the trade-off for this massive speedup, we evaluate the optimality gap (visualized in the bottom panels of Figure~\ref{fig:performance_analysis}). To provide a rigorous assessment of solution quality, we calculate the True Optimality Gap. Unlike relative deviations between two feasible solutions, this metric compares the heuristic solution ($Z_{Heur}$) against the theoretical lower bound ($Z_{LB}$) established by the exact solver:
\begin{equation}
    Gap = \frac{Z_{Heur} - Z_{LB}}{Z_{Heur}} \times 100\%
\end{equation}
The heuristic maintains high solution quality, with a median gap of just 4.70\% for the challenging $4\times4$ layout with 2 AMRs. As expected, the gap increases with fleet size and exhibits a higher sensitivity to the unit-load-to-slot-ratio. Medians typically range between 10\% and 20\% for larger configurations, reflecting the heuristic's tendency to prioritize rapid conflict resolution over finding the mathematically perfect spatial sequence.

\paragraph{Constraint Handling and Temporal Flexibility}
A key distinction in this assessment lies in the treatment of time windows. The EF strictly enforces the start times at the storage slots — which are derived from the external deadlines in \eqref{eq:tw_retrieval} — as hard constraints, classifying any temporal deviation as an infeasible solution. In contrast, the heuristic is designed to maintain operational flow by selectively relaxing these internal bounds. Specifically, it permits internal storage tasks to complete later and retrieval tasks to commence earlier than originally scheduled. This approach ensures that the handovers at the Source and Sink remain perfectly synchronized with external processes, while providing the internal buffer operations with the temporal flexibility necessary to resolve spatial conflicts. Among the generated solutions, approximately 9.5\% utilized this flexibility. While strictly speaking suboptimal compared to the rigid EF lower bound, these solutions are operationally superior in a production context, as they preserve the punctuality of external processes while preventing internal system lockups.

\subsection{Qualitative Analysis of Solution Behavior}
To validate the performance of the proposed approach, we analyze the operational behavior of the generated solutions in two distinct scenarios: reactive conflict resolution in high-density confined spaces and proactive capacity management in large-scale brownfield layouts. This analysis examines solutions for $8\times3$, $5\times5$, and $6\times6$ layouts, as well as a real-world brownfield case from the large-scale instance set.

\begin{figure}[htbp!]
    \centering
    \includegraphics[width=\textwidth]{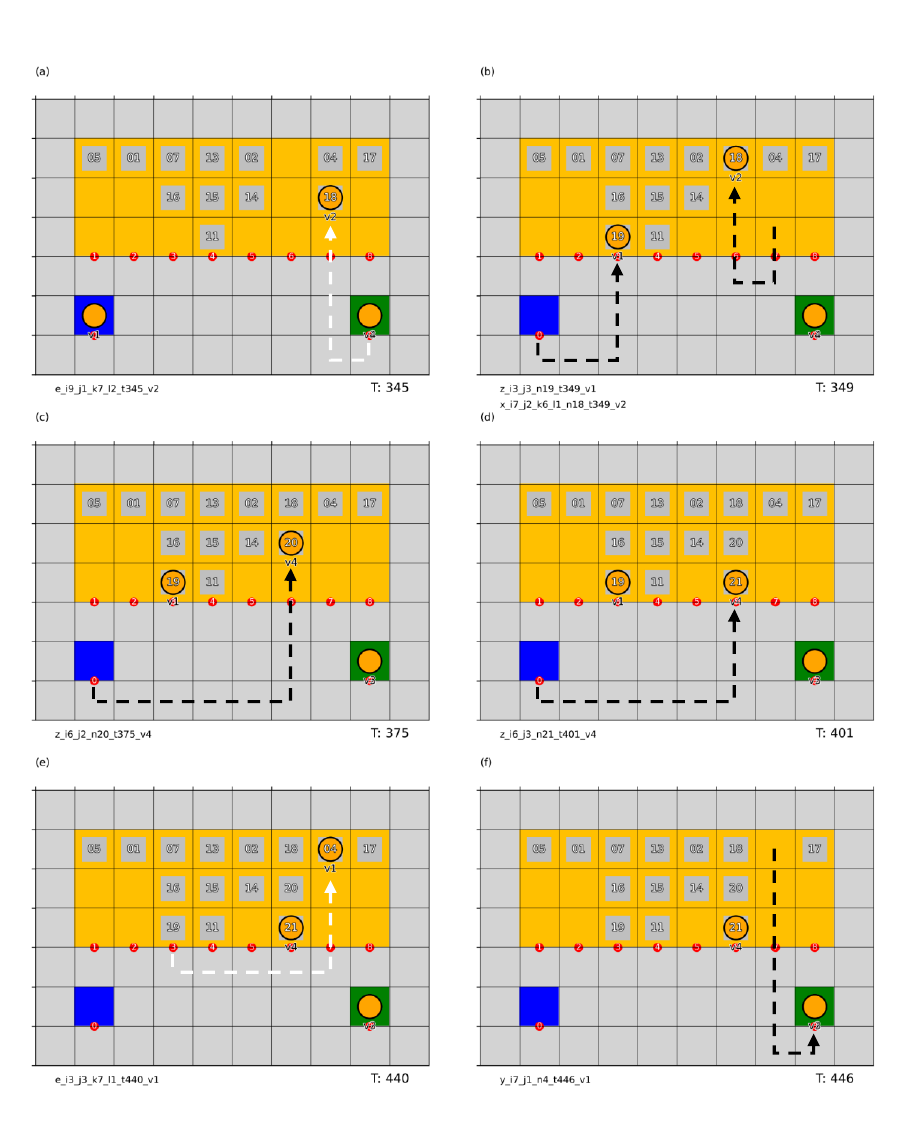}
    \caption{Visualization of a solution in a high-density $8\times3$ layout (21 unit loads) with 4 AMRs. The bottom-left of each sub-figure indicates the corresponding decision variable from the EF (also visualized with the dotted arrows; white for empty drives; black for transportation of an UL), while T denotes the discrete time step of the event.}
    \label{fig:solution_sequence}
\end{figure}

\subsubsection{Reactive Conflict Resolution ($8\times3$ Layout)}
First, we examine a highly constrained environment serviced by a fleet of 4 AMRs (Utilization $\approx 90\%$). Figure \ref{fig:solution_sequence} visualizes a complex interleaved reshuffling and storage sequence spanning from time step $t=345$ to $t=446$. The system's objective is to retrieve the target unit load (UL) 04, which is currently blocked.
\begin{enumerate}
    \item \textbf{Reshuffling ($t=345-349$):} The heuristic identifies UL 18 as blocking the retrieval of the target UL 04. AMR V2 is assigned to reshuffle UL 18 to a temporary position.
    \item \textbf{Storage ($t=375-401$):} Concurrently, new storage requests for UL 20 and UL 21 arrive. Rather than blocking the now-accessible lane for UL 04, the heuristic utilizes the space in front of the reshuffled UL 18. It directs the fleet to stack UL 20 ($t=375$) and UL 21 ($t=401$) in front of UL 18. This choice avoids creating new blockages, as UL 20 and UL 21 are scheduled for retrieval earlier than UL 18.
    \item \textbf{Retrieval ($t=440-446$):} With the blockage removed and incoming traffic diverted to non-critical lanes, AMR V1 navigates to the target lane and retrieves UL 04  within its designated retrieval window.
\end{enumerate}

Overall, this scenario exemplifies how the heuristic resolves conflicts by leveraging temporal slack. It stacks UL 20/21 in front of UL 18 without creating blockages; this is permissible due to the specific retrieval order and maximizes the utility of the constrained floor space.

\subsubsection{Spatial Strategies in Brownfield Layouts}
To demonstrate the algorithm's capability to operate within complex layouts, we applied the heuristic to a real-world layout from the surface coating industry (Figure \ref{fig:real_world_case}). 

This scenario uses the irregular floor area surrounding a fixed high-gloss coating machine (grey obstacle). The layout is characterized by the adjacent positioning of the Source and Sink (bottom right), requiring all AMR movements to be sequenced through the central aisle.

Figure \ref{fig:real_world_operation} illustrates the system state at $t=012$. The heuristic utilizes this layout through two key mechanisms:
\begin{enumerate}
    \item \textbf{Static Lanes:} The irregular floor plan is manually decomposed into a set of static Lanes. Specifically, the Orange and Yellow zones are mapped to single-deep lanes (depth 1), while the Green zone forms double-deep lanes (depth 2). Static lanes allow the algorithm to apply standard stack-based logic to handle LIFO constraints, regardless of the specific physical arrangement.
    \item \textbf{Look-Ahead Placement:} The heuristic optimizes storage locations based on retrieval deadlines. For instance, although both UL 19 and UL 20 are retrieved late in the horizon, the heuristic distinguishes between them. UL 20, having the later deadline, is moved to a distant location (lane 8) during initial reshuffling, whereas UL 19 is kept in a closer slot (lane 16). This prioritization ensures that accessible buffer capacity is preserved for unit loads with earlier retrieval times.
\end{enumerate}
These results show exemplarly that the proposed logic enables robust operations in layout-constrained environments, preventing deadlocks through the temporal sequencing of lane access.

\begin{figure}[htbp]
    \centering
    \begin{subfigure}[b]{0.48\textwidth}
        \centering
        \includegraphics[width=\textwidth]{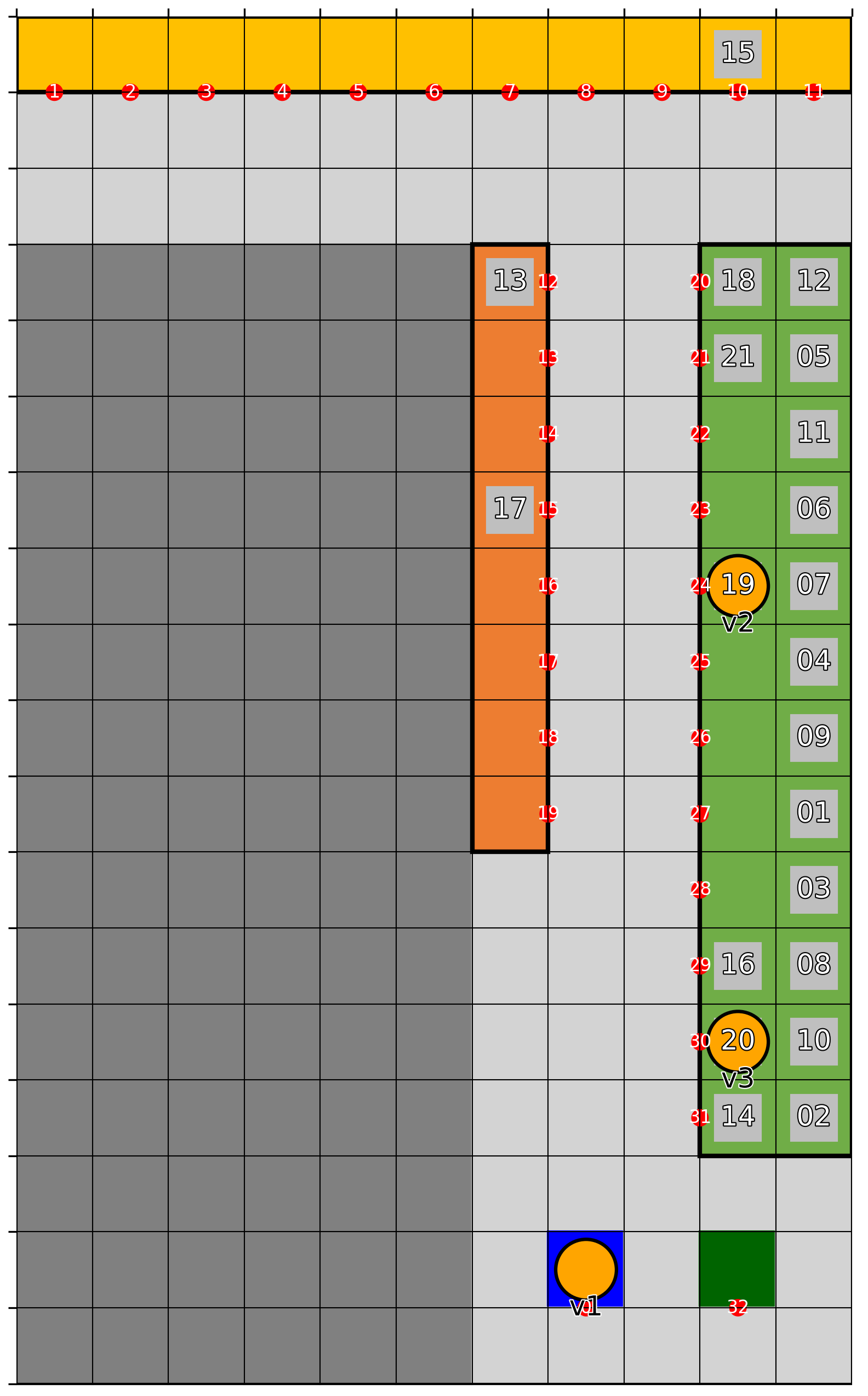} 
        \caption{\textbf{Topology ($t=001$):} The layout utilizes the residual space around a high-gloss coating machine (grey) with adjacent Source (blue) and Sink (green). The niches are manually abstracted into static lanes with varying depths: orange/yellow zones (depth 1) and green zones (depth 2). \\}
        \label{fig:real_world_layout}
    \end{subfigure}
    \hfill
    \begin{subfigure}[b]{0.48\textwidth}
        \centering
        \includegraphics[width=\textwidth]{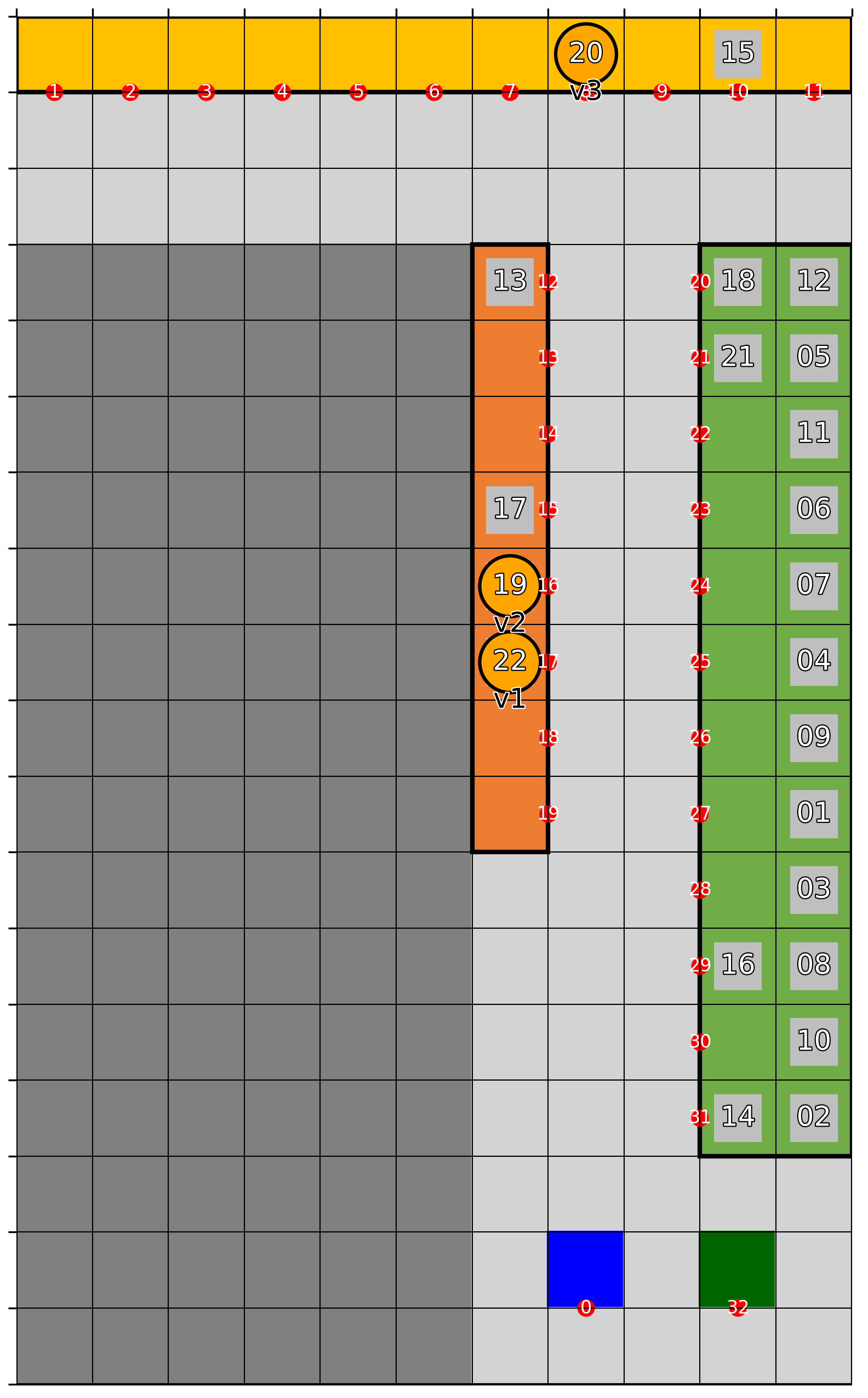} 
        \caption{\textbf{Operational State ($t=012$):} Early optimization. The snapshot captures a reshuffling operation where AMR V2 relocates unit load UL 19 to the single-deep yellow zone (lane 16). In contrast, the lower-priority UL 20 is moved to the distant yellow zone (lane 8), preserving closer slots for urgent unit loads.}
        \label{fig:real_world_operation}
    \end{subfigure}
    \caption{Real-world use case (surface coating). This brownfield scenario demonstrates the conversion of dead floor space into an autonomous buffer. Unlike open grids, the layout is dictated by the footprint of the machinery. The visualization highlights the heuristic's capability to assign inventory to separated zones (depth 1 vs. depth 2) to optimize storage utilization within the irregular boundaries.}
    \label{fig:real_world_case}
\end{figure}

\subsubsection{Spatial Decision Policy}
To analyze the spatial allocation strategies of the heuristic, we evaluate the occupancy intensity and traffic density within the brownfield scenario. As illustrated in Figure \ref{fig:heatmaps}, occupancy intensity is measured by the average number of timesteps a storage slot holds a unit load, whereas traffic density quantifies the cumulative frequency of AMR visits at each lane access point. This evaluation reveals that the algorithm adopts a highly strategic utilization of the irregular space without explicit pre-programming. Rather than treating the available floor as a rigid grid, the system behaves organically, dynamically adapting its spatial footprint to the current operational load.

\begin{figure}[htbp]
    \centering
    \begin{subfigure}[b]{0.48\textwidth}
        \centering
        \includegraphics[width=\textwidth]{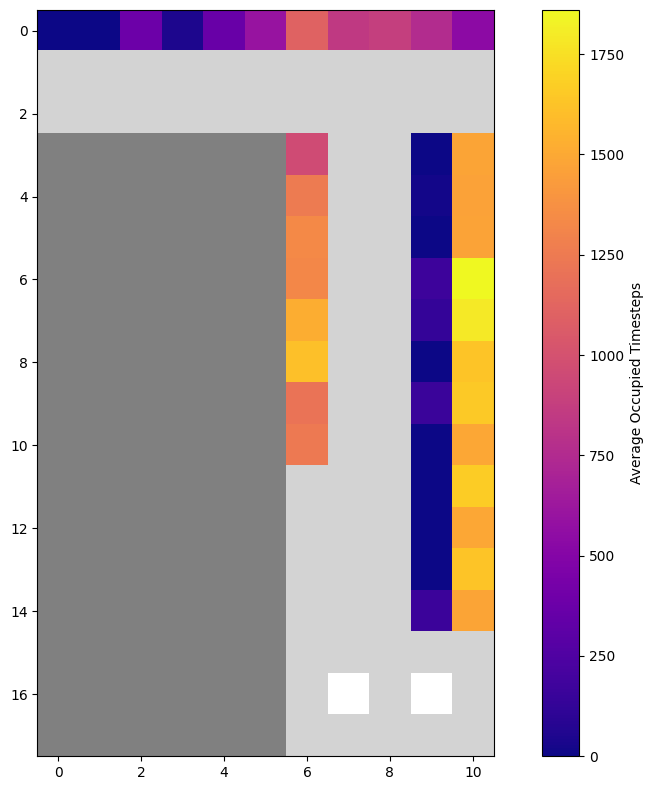}
        \caption{\textbf{Occupancy Intensity:} Average duration a slot is occupied. Lighter colors indicate the slot is utilized most of the time.}
        \label{fig:heatmap_occ}
    \end{subfigure}
    \hfill
    \begin{subfigure}[b]{0.46\textwidth}
        \centering
        \includegraphics[width=\textwidth]{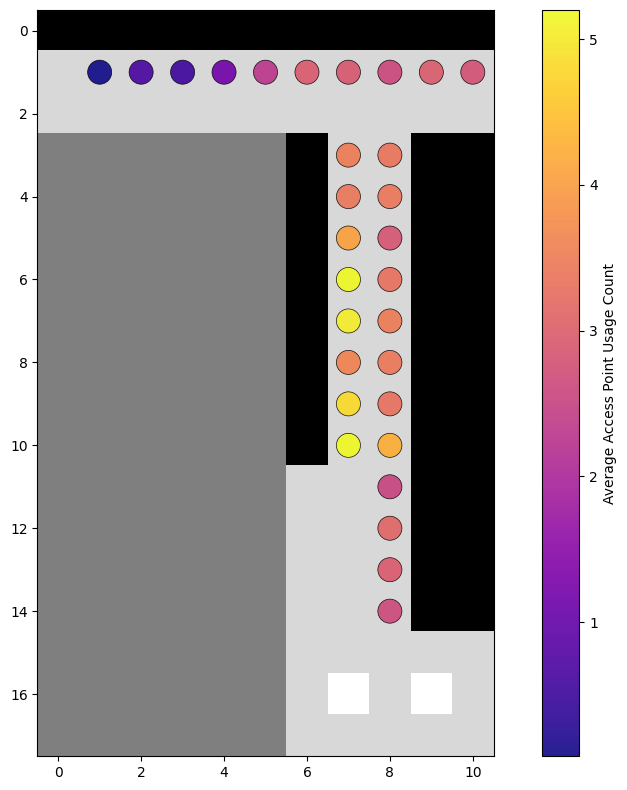}
        \caption{\textbf{Traffic Density:} Frequency of AMR visits at access points. Lighter colors indicate the access point is utilized frequently.}
        \label{fig:heatmap_traff}
    \end{subfigure}
    \caption{Algorithmic storage location assignment in the brownfield scenario. The heatmaps contrast buffer duration (a) with traffic flow (b), highlighting the differentiation between high-density storage and transfer zones.}
    \label{fig:heatmaps}
\end{figure}

The heatmaps reveal three key behaviors. Notably, these advanced spatial strategies emerge despite the methodological simplifications of the approach. Specifically, the heuristic nature of the $A^*$ search and the hierarchical decomposition of the problem. This demonstrates that the designed cost function captures global system dynamics, leading to the following emergent behaviors:

\begin{enumerate}
    \item \textbf{Lane-Blocking Avoidance:} A feature in Figure \ref{fig:heatmap_occ} is the low utilization of the front positions in double-deep lanes, despite their proximity to the access point. The heuristic demonstrates a preference for utilizing distant single-deep slots rather than blocking an occupied rear slot in a closer lane. This confirms that the algorithm prioritizes accessibility over physical proximity, accepting longer travel distances to prevent future reshuffling operations.
    
    \item \textbf{Buffer Duration-Based Inventory Stratification:} The system exhibits a distinct separation of inventory based on buffer duration in the buffer. Deep slots in deep lanes show the highest cumulative occupancy (Figure \ref{fig:heatmap_occ}), indicating they are used to buffer unit loads with late deadlines. Conversely, slots adjacent to the Source and Sink exhibit the highest traffic density (Figure \ref{fig:heatmap_traff}) but comparatively low occupancy intensity. The heuristic adaptively treats these prime locations as a staging area with short buffer durations for immediate retrieval tasks.
    
    \item \textbf{Capacity-Adaptive Breathing Topology:} The algorithm automatically adjusts its active storage area based on current inventory levels. During periods of low utilization, it allocates unit loads to nearby slots, minimizing AMR travel distances and leaving distant areas empty. As storage density increases, the algorithm dynamically expands the active storage area into the further distant slots. This ensures strict travel time optimization under normal conditions, while unlocking maximum capacity during peak congestion.
\end{enumerate}

These behaviors suggest that the heuristic is capable of highly adaptive, context-aware decision making on complex floor plans.

\subsection{Layout Sensitivity and Saturation Points}
To isolate the impact of layout topology from pure capacity, we expanded the evaluation to compare rectangular configurations ($8\times3$) against dense square blocks ($5\times5$, $6\times6$) and irregular real-world layouts (e.g. Figure \ref{fig:real_world_case}). These scenarios were evaluated across fleet sizes ranging from $|\mathcal{V}|=2$ to $6$ AMRs. Since feasibility rates remained consistent regardless of fleet density — confirming the heuristic's deadlock avoidance — we aggregate these results to focus on the more decisive structural factors: number of access points and lane-to-depth ratio.

Table \ref{tab:fill_rate_sensitivity} presents a sensitivity analysis, breaking down feasibility by the unit-load-to-slot-ratio and access constraints. The table lists the absolute number of solved instances versus the total number of generated instances for each configuration. The results highlight three critical findings regarding system stability:

\begin{table}[htbp]
\centering
\caption{Sensitivity Analysis: feasibility by unit load to slot ratio and access directions. The table displays the success rate as a fraction (Solved/Generated) for various layouts.}
\label{tab:fill_rate_sensitivity}
\resizebox{0.7\textwidth}{!}{%
\begin{tabular}{lcccccc}
\toprule
\multirow{2}{*}{\textbf{Layout Topology}} & \multirow{2}{*}{\textbf{Ratio}} & \multicolumn{4}{c}{\textbf{Access Directions}} & \textbf{Solve Rate} \\
\cmidrule(lr){3-6}
 & & \textbf{1} & \textbf{2} & \textbf{3} & \textbf{4} & \\
\midrule
\multirow{4}{*}{Rectangle ($8\times3$)} 
 & 0.7 & 30/50 & 50/50 & -- & -- & 80.0\% \\
 & 0.8 & 15/50 & 50/50 & -- & -- & 65.0\% \\
 & 0.9 & 5/50 & 50/50 & -- & -- & 55.0\% \\
 & 1.0 & 0/50 & 45/50 & -- & -- & 45.0\% \\
\midrule
\multirow{4}{*}{Block ($5\times5$)} 
 & 0.7 & 0/50 & 85/100 & -- & 50/50 & 67.5\% \\
 & 0.8 & 0/50 & 70/100 & -- & 50/50 & 60.0\% \\
 & 0.9 & 0/50 & 20/100 & -- & 45/50 & 32.5\% \\
 & 1.0 & 0/50 & 10/100 & -- & 40/50 & 25.0\% \\
\midrule
\multirow{4}{*}{Block ($6\times6$)} 
 & 0.7 & 0/50 & 15/100 & -- & 50/50 & 32.5\% \\
 & 0.8 & 0/50 & 0/100 & -- & 50/50 & 25.0\% \\
 & 0.9 & 0/50 & 5/100 & -- & 25/50 & 15.0\% \\
 & 1.0 & 0/50 & 0/100 & -- & 25/50 & 12.5\% \\
\midrule
\multirow{4}{*}{Real-World (39 and 43 slots)} 
 & 0.5--0.8 & -- & -- & 450/450 & -- & 100.0\% \\
 & 0.9 & -- & -- & 98/100 & -- & 98.0\% \\
 & 1.0 & -- & -- & 75/100 & -- & 75.0\% \\
\bottomrule
\end{tabular}%
}
\end{table}

Most notably, the coating industry layouts with high slot counts demonstrate exceptional stability compared to the theoretical block models. As shown in Table \ref{tab:fill_rate_sensitivity}, feasibility remains at or near 100\% up to a unit load-to-slot ratio of 0.9. Even at full capacity, the system successfully solves between 70\% and 80\% of instances. The decline in solvability for these instances is attributable to two converging factors: primarily, the time required for deep reshuffling exceeds the theoretical baselines derived from simplified travel estimates, rendering some instances structurally infeasible. Additionally, for theoretically solvable but complex instances, the A* search (Stage 2) becomes a computational bottleneck, hitting time limits due to the search space explosion.

The analysis of the $8\times3$ layout confirms that the topology dictates performance. Even at a saturation ratio of 1.0, this configuration maintains 90\% feasibility provided it allows for 2-sided access (45/50 solved). This evidence suggests that a high number of independent access points facilitates parallel operations, effectively mitigating the congestion effects typically associated with high storage density.

In sharp contrast, deep square layouts (e.g., $6\times6$) exhibit significantly earlier saturation points. With limited access (1--2 sides), performance degrades rapidly even at moderate ratios (0.7). This confirms that for deep storage, optimization cannot fully compensate for a lack of reshuffling space; consequently, either lower storage densities or full perimeter access (4 sides) are requisite for reliable solvability.

\subsection{Managerial Insights}
\label{subsec:managerial_insights}

The computational experiments provide guidelines for the design and operation of autonomous buffer zones. By synthesizing the results, we derive four principles for managing floor storage.

\subsubsection{Layout Fragmentation and Access Efficiency}
In brownfield facilities, contiguous areas for buffering are rarely available. The results on layout topology show that spatial fragmentation is not a disadvantage. Small, distributed rectangular buffers (e.g., $8\times3$) outperform large, symmetrical block layouts with deeper lanes (e.g., $6\times6$), even if these distributed zones are located further away from the source and sink.
\\
\textit{Strategic Insight:} Buffering does not require consolidating space into a single zone. Utilizing decentralized floor space near production lines is effective, provided these pockets maintain high access availability relative to storage slots. Distributed buffers with multiple access directions prevent bottlenecks more effectively than deep storage blocks.

\subsubsection{Capacity-Adaptive Breathing Topology}
Traditional material handling relies on static zoning, which requires manual intervention. Heatmap analysis shows that the heuristic autonomously performs storage location assignment without explicit pre-configuration.
\\
\textit{Strategic Insight:} Despite the limitations of the $A^*$ search and the decomposition, our heuristic enables capacity-adaptive breathing topology by placing urgent unit loads at the periphery and less critical loads in deeper positions. This logic reorganizes the buffer in real-time based on the production schedule, aiming to optimize the accessibility for the next retrieval window while the utilized area expands or contracts based on current demand.

\subsubsection{The 90\% Stability Threshold}
The sensitivity analysis shows a stability threshold across all tested layouts. Solvability and system reliability decline when the unit-load-to-slot ratio exceeds 0.9. At higher ratios, the system lacks the unoccupied slots required for reshuffling.
\\
\textit{Strategic Insight:} Planners must distinguish between storage capacity and operational throughput capacity. A production buffer requires approximately 10\% of slots as slack to resolve deadlocks. Operating above this 90\% threshold transforms the buffer into a static storage yard, rendering it brittle to stochastic production requests.

\subsubsection{Operational Robustness and Fleet Scalability}
The algorithm maintains robustness across varying fleet sizes from 2 to 6 robots without parameter retuning. The system utilizes additional robots to reduce makespan until the maximum throughput capacity is reached.
\\
\textit{Strategic Insight:} This ensures operational robustness through fleet scalability. If a robot is removed for maintenance, the algorithm redistributes tasks and adapts the throughput capacity without causing an operational standstill. Process continuity is decoupled from the specific fleet count, providing investment security.

%% file: sections/conclusion.tex
\section{Conclusion and Future Work}
\label{sec:conclusion}

This paper addressed the Multi-AMR Buffer Storage, Retrieval, and Reshuffling Problem (BSRRP) by establishing an Exact Formulation (EF) for benchmarking and proposing a hierarchical heuristic. Our research demonstrates that automating dense floor storage in brownfield environments requires the integration of reshuffling logic and multi-AMR coordination, as implemented in our two-stage approach.

The computational experiments reveal a disparity in tractability between exact and heuristic methods. While exact approaches often fail to converge for dense industrial scenarios, the hierarchical heuristic achieves speedup factors ranging from 450x to over 3,000x compared to the exact solver.

Regarding operational robustness, our analysis identified a saturation threshold at a unit load to slot ratio of 0.9. Beyond this ratio, the lack of reshuffling space renders the system brittle, confirming that 10\% of capacity must be treated as slack to resolve deadlocks. Furthermore, strict time windows were identified as a source of infeasibility; treating deadlines as soft constraints allowed the heuristic to resolve high-traffic scenarios where the exact solver failed, confirming that flow continuity must take precedence over precision in time-critical environments.

Finally, the physical layout topology and the heuristic's capacity-adaptive behavior play a decisive role. Our experiments demonstrate that rectangular layouts with high access availability outperform deep square configurations by enabling parallel access. The heuristic adaptively adjusts the storage footprint based on current inventory levels: it utilizes nearby slots during low demand to minimize travel distances and expands into distant slots only when capacity requirements increase. This allows the system to optimize storage depth based on buffer duration without explicit pre-configuration, adapting autonomously to irregular brownfield constraints. This confirms that a decoupled approach—combining $A^*$ search for task sequencing with Constraint Programming for scheduling—is a viable path for industrial control.

Operationally, future work will focus on integrating kinematic constraints (e.g., acceleration profiles, turning radii) into the scheduling logic. Additionally, implementing post-processing trajectory smoothing could further enhance path efficiency. 

Algorithmically, profiling identifies the $A^*$ search for task sequencing as the primary bottleneck in dense scenarios. Future optimization of this component — through pre-computed pattern databases or Multi-Agent Path Finding (MAPF) refinements — would eliminate remaining latency. Furthermore, this modular architecture establishes a data-driven infrastructure for online operations. The decoupled solver can serve as a baseline for real-time decision-making, where the system must continuously adapt to dynamic production streams and stochastic arrival rates in a data-rich environment.

%% file: sections/acknowledgments.tex
\section*{Acknowledgements} 

This research was presented as a work-in-progress at the VeRoLog conference 2025.

The authors acknowledge the use of artificial intelligence (AI) tools during the preparation of this manuscript.
Specifically, Gemini (versions 1.5 to 3.0; Google) and GitHub Copilot (Claude Sonnet 4.0 and 4.5) were used to help with
aspects of software development, language improvement and editing, and the classification of the literature
relevant to this study. The authors have reviewed and edited all
content and assume full responsibility for the accuracy, integrity, and originality of the work presented. 

\section*{Disclosure statement}
The authors declare no conflicts of interest.

\section*{Funding}
This work was supported by the European Union - Next GenerationEU under Grant 13IK0321

\section*{Data and Code availability statement}
The source code is available under \url{https://github.com/mdisse/BSRRP_IP_Heuristic}. Instances and solutions are available under \url{https://doi.org/10.5281/zenodo.19222950}. 

%% file: sections/appendix_model.tex
\section{Multi-AMR Buffer Storage, Retrieval, and Reshuffling Problem - Complete Model}
\label{sec:appendix_model}

\begin{table}[ht]
\small
\centering
    \begin{tabular}{l l p{8.5cm}}
    \toprule
    \textbf{Element} & \textbf{Notation} & \textbf{Description} \\
    \midrule
    \multicolumn{3}{l}{\textit{Sets and Indices}} \\
    \midrule
    Set of unit loads & $\mathcal{N}$, $n$ & Range: $\{1, 2, \dots, N\}$ \\
    Set of time steps & $\mathcal{T}$, $t$ & Range: $\{1, 2, \dots, T\}$ \\
    Set of all lanes & $\mathcal{I}$, $i, k$ & Range: $\{0, 1, \dots, I\}$ \\
    Set of buffer lanes & $\mathcal{I'}$ & Range: $\{1, 2, \dots, I - 1\}$ (excluding sink and source) \\
    Set of slots & $\mathcal{J}_{i,k}$, $j, l$ & Range: $\{1, 2, \dots, J\}$ \\
    Set of vehicles & $\mathcal{V}$, $v$ & Range: $\{1, 2, \dots, V\}$ \\
    \midrule
    \multicolumn{3}{l}{\textit{General Notation and Parameters}} \\
    \midrule
    Unit load & $(u_n, r_n, a_n)$ & A unit load with label $u_n$. Retrieval window opens at $r_n$, arrival window at $a_n$. \\
    Slot & $[i,j]$ & Slot in the $i^{th}$ lane and $j^{th}$ slot (deepest $j = 1$, at the perimeter $j = J$). \\
    Outermost slot & $[i, J_i]$ & The outermost slot $J_i$ in lane $i$. \\
    Sink & $[I, 1] = [I, J_I]$ & Modelled as a slot in the lane with the highest value for $i$. \\
    Source & $[0, 1] = [0, J_I]$ & Modelled as a slot in the lane with the lowest value for $i$. \\
    Relocation move & $[i,j] \rightarrow [k,l]$ & A relocation move from slot $[i,j]$ to $[k,l]$. \\
    Retrieval move & $[i,j] \rightarrow [I, J]$ & Retrieving a unit load from slot $[i,j]$. \\
    Empty drive & $[i,j] \hookrightarrow [k,l]$ & Driving the AMR from slot $[i,j]$ to $[k,l]$ without transporting a load. \\
    Storage move & $[0,1] \rightarrow [i, j]$ & Storing a unit load to slot $[i,j]$. \\
    Retrieval window & $[r_n, r_n+\rho_n]$ & Time window for retrieval of $u_n$; $\rho_n$ is the maximum delay. \\
    Arrival window & $[a_n, a_n+\alpha_n]$ & Time window for arrival of $u_n$; $\alpha_n$ is the maximum delay. \\
    Distance & $d_{ijkl}$ & Distance between the slots $[i, j]$ and $[k, l]$. \\
    Travel time & $\tau_{ijkl}$ & Travel time (incl. handling time) between the slots $[i, j]$ and $[k, l]$. \\
    \bottomrule
    \end{tabular}
    \caption{Sets, indices, and general notation for the Integer Programming model}
    \label{tab:combined_notation}
\end{table}

\scriptsize
\begin{align}
b_{ijnt} &= \begin{cases}
    1 & \text{if unit load } n \text{ is in } [i,j] \text{ at time } t,\\
    0 & \text{otherwise;}
\end{cases} \\
& \nonumber \hspace*{3.65em} \forall i \in \mathcal{I}, \forall j \in \mathcal{J}_i, \forall n \in \mathcal{N}, \forall t \in \mathcal{T}& \\
x_{ijklntv} &= \begin{cases}
    1 & \text{if unit load } n \text{ is relocated from } [i,j] \text{ to } [k,l] \text{ at time } t \text{ by AMR } v, \\
    0 & \text{otherwise;}
\end{cases}  \\
& \nonumber \hspace*{3.65em} \forall i, k \in \mathcal{I'}, \forall j \in \mathcal{J}_i, \forall l \in \mathcal{J}_k, \forall n \in \mathcal{N} , \forall t \in \mathcal{T}, \forall v \in \mathcal{V}&  \\
y_{ijntv} &= \begin{cases}
    1 & \text{if unit load } n \text{ is retrieved from } [i,j] \text{ at time } t \text{ by AMR } v,\\
    0 & \text{otherwise;}
\end{cases} \\
& \nonumber \hspace*{3.65em} \forall i \in \mathcal{I}\setminus \{I\}, \forall j \in \mathcal{J}_i, \forall n \in \mathcal{N} , \forall t \in \mathcal{T}, \forall v \in \mathcal{V}&  \\
g_{nt} &= \begin{cases}
    1 & \text{if unit load } n \text{ has been retrieved at time } t' \in \{1, \dots, t-1\}, \\
    0 & \text{otherwise;}
\end{cases} \\
& \nonumber \hspace*{3.65em} \forall n \in \mathcal{N}, \forall t \in \mathcal{T}& \\
z_{ijntv} &= \begin{cases}
    1 & \text{if unit load } n \text{ is stored in } [i,j] \text{ at time } t \text{ by AMR } v,\\
    0 & \text{otherwise;}
\end{cases} \\
& \nonumber \hspace*{3.65em} \forall i \in \mathcal{I'}, \forall j \in \mathcal{J}_i, \forall n \in \mathcal{N} , \forall t \in \mathcal{T}, \forall v \in \mathcal{V}&  \\
s_{nt} &= \begin{cases}
    1 & \text{if unit load } n \text{ has been stored at time } t' \in \{1,\dots,t-1\}, \\
    0 & \text{otherwise;}
\end{cases} \\
& \nonumber \hspace*{3.65em} \forall n \in \mathcal{N}, \forall t \in \mathcal{T}& \\
e_{ijkltv} &= \begin{cases}
    1 & \text{if the AMR drives from slot } [i,j] \text{ to slot } [k,l] \text{ at time } t, \\
      & \text{without transporting a unit load} \\
    0 & \text{otherwise;}
\end{cases} \\
& \nonumber \hspace*{3.65em} \forall i, k \in \mathcal{I}, \forall j \in \mathcal{J}_i, \forall l \in \mathcal{J}_k, \forall t \in \mathcal{T}, \forall v \in \mathcal{V}&  \\
c_{ijtv} &= \begin{cases}
    1 & \text{if the AMR } v \text{ is at slot } [i,j] \text{ at time } t, \\
    0 & \text{otherwise;}
\end{cases} \\
& \nonumber \hspace*{3.65em} \forall i \in \mathcal{I}, \forall j \in \mathcal{J}_i, \forall t \in \mathcal{T}, \forall v \in \mathcal{V}&  
\end{align}

\begin{equation}
    T = \max_{n \in \mathcal{N}, m \in \mathcal{N}, d_{n}<\infty} \{a_{n} + \alpha_{n}, d_{m} + \delta_{n}\} 
\end{equation}

\begin{equation}
    \tau_{ijkl} = \begin{cases}
         \max(1, d_{ijkl}) & \text{if performing an empty drive } e_{ijklt}, \\ 
         \max(1, d_{ijkl} + 2h) & \text{otherwise;}
    \end{cases}
\end{equation}

\scriptsize
\begin{align}
    &\text{\textbf{Starting Constraints}} & \nonumber     \\
    b_{ijn1} &= \begin{cases}
        1, & \text{if unitload } n \text{ starts in slot } [i,j] \\
        0, & \text{otherwise}
    \end{cases} & \quad\quad \forall i \in \mathcal{I'}, \forall j \in \mathcal{J}_i, \forall n \in \mathcal{N} \\
    s_{n1} &= \begin{cases}
        1, & \text{if unitload is initially stored in buffer zone} \\
        0, & \text{otherwise}
    \end{cases} & \quad\quad \forall n \in \mathcal{N} \\
    c_{ij1v} &= \begin{cases}
        1, & \text{if vehicle } v \text{ starts in slot } [i,j] \\ 
        0, & \text{otherwise}
    \end{cases} & \quad\quad \forall i \in \mathcal{I}, \forall j \in \mathcal{J}_i, \forall v \in \mathcal{V} 
\end{align}

\tiny
\setlength{\jot}{-1pt}
\begin{align}
    &\text{\textbf{Objective function}} & \nonumber     \\
    &\min \sum_{i \in \mathcal{I}}  \sum_{j \in \mathcal{J}_i}  \sum_{t \in \mathcal{T}} \sum_{v \in \mathcal{V}} \Bigl(  \sum_{n \in \mathcal{N}} \bigl( y_{ijntv} * d_{ijI1} + z_{ijntv} * d_{01ij} \bigr) \\
    &\nonumber \quad\quad\quad\quad + \sum_{k \in \mathcal{I}} \sum_{l \in \mathcal{J}_k} \bigl( \sum_{n \in \mathcal{N}} x_{ijklntv} + e_{ijkltv} \bigr) * d_{ijkl} \Bigr) \\
    &\text{\textbf{Subject to:}} & \nonumber     \\
    &\sum_{i \in \mathcal{I'}} \sum_{j \in \mathcal{J}_i} b_{ijnt} \leq s_{nt}, \quad\quad \forall n \in \mathcal{N}, \forall t \in \mathcal{T} \\
    &g_{nt} \geq s_{nt} - \sum_{i \in \mathcal{I'}} \sum_{j \in \mathcal{J}_i} b_{ijnt}, \quad\quad \forall n \in \mathcal{N}, \forall t \in \mathcal{T} \\
    &\sum_{n \in \mathcal{N}} b_{ijnt} \leq 1, \quad\quad \forall i \in \mathcal{I'}, \forall j \in \mathcal{J}_i, \forall t \in \mathcal{T} \\ 
    &\sum_{n \in \mathcal{N}} b_{ijnt} \geq \sum_{n \in \mathcal{N}} b_{ij+1nt}, \quad\quad \forall i \in \mathcal{I'}, \forall j \in \mathcal{J}_i \setminus J_i, \forall t \in \mathcal{T} \\
    &\sum_{k \in \mathcal{I'}}  \sum_{l \in \mathcal{J}_k} \sum_{n \in \mathcal{N}} x_{ijklntv} + \sum_{k \in \mathcal{I}}  \sum_{l \in \mathcal{J}_k} e_{ijkltv} + \sum_{n \in \mathcal{N}} y_{ijntv} \leq c_{ijtv}, \\
    &\nonumber \quad\quad\quad\quad \forall i \in \mathcal{I'}, \forall j \in \mathcal{J}_i, \forall t \in \mathcal{T}, \forall v \in \mathcal{V} \\
    &\sum_{k \in \mathcal{I}}  \sum_{l \in \mathcal{J}_k} e_{I1kltv} \leq c_{I1tv}, \quad\quad \forall t \in \mathcal{T}, \forall v \in \mathcal{V} \\
    &\sum_{k \in \mathcal{I'}} \sum_{l \in \mathcal{J}_k} z_{klntv}  + \sum_{k \in \mathcal{I}} \sum_{l \in \mathcal{J}_k} e_{01kltv} + y_{01ntv} \leq c_{01tv}, \quad\quad \forall t \in \mathcal{T}, \forall v \in \mathcal{V} \\
    &\sum_{v \in \mathcal{V}} y_{01ntv} + s_{nt} \leq 1, \quad\quad \forall n \in \mathcal{N}, \forall t \in \mathcal{T} \\
    &\sum_{v \in \mathcal{V}} \bigl( y_{ijntv} + \sum_{k \in \mathcal{I'}} \sum_{l \in \mathcal{J}_k} x_{ijklntv} \bigr) \leq b_{ijnt}, \quad\quad \forall i \in \mathcal{I'}, \forall j \in \mathcal{J}_i, \forall n \in \mathcal{N}, \forall t \in \mathcal{T}\\
    &b_{ijnt} = b_{ijn(t-1)} + \sum_{v \in \mathcal{V}} \Bigl[ \sum_{k \in \mathcal{I'}} \sum_{l \in \mathcal{J}_k}  \Bigl( x_{klijn(t-\tau_{klij})v} -  x_{ijkln(t-1)v} \Bigr) - y_{ijn(t-1)v} + z_{ijn(t-\tau_{01ij})v} \Bigr], \\ 
    &\nonumber \quad\quad\quad\quad \forall i \in \mathcal{I'}, \forall j \in \mathcal{J}_i, \forall n \in \mathcal{N}, \forall t \in \mathcal{T} \setminus 1 \\
    &g_{nt} = \sum_{i \in \mathcal{I}\setminus I} \sum_{j \in \mathcal{J}_i} \sum_{t'=1}^{t-1} \sum_{v \in \mathcal{V}} y_{ijnt'v}, \quad\quad \forall n \in \mathcal{N}, \forall t \in \mathcal{T} \\
    &s_{nt} = s_{n1} + \sum_{i \in \mathcal{I'}} \sum_{j \in \mathcal{J}_i} \sum_{t'=1}^{t-\tau_{01ij}} \sum_{v \in \mathcal{V}} z_{ijnt'v}, \quad\quad \forall n \in \mathcal{N}, \forall t \in \mathcal{T} \\
    &c_{ijtv} = c_{ij(t-1)v} + \sum_{n \in \mathcal{N}}  z_{ijn(t-\tau_{01ij})v} - \sum_{n \in \mathcal{N}} y_{ijn(t-1)v} + \sum_{k \in \mathcal{I'}} \sum_{l \in \mathcal{J}_k} \sum_{n \in \mathcal{N}} x_{klijn(t-\tau_{klij})v} \\ 
    &\nonumber \quad\quad\quad\quad + \sum_{k \in \mathcal{I}} \sum_{l \in \mathcal{J}_k} e_{klij(t-\tau_{klij})v} - \sum_{k \in \mathcal{I'}} \sum_{l \in \mathcal{J}_k} \sum_{n \in \mathcal{N}} x_{ijkln(t-1)v} - \sum_{k \in \mathcal{I}} \sum_{l \in \mathcal{J}_k} e_{ijkl(t-1)v}, \\
    &\nonumber \quad\quad\quad\quad \forall i \in \mathcal{I'}, \forall j \in \mathcal{J}_i, \forall t \in \mathcal{T} \setminus 1, \forall v \in \mathcal{V} \\
    &c_{I1tv} = c_{I1(t-1)v} + \sum_{i \in \mathcal{I}\setminus I} \sum_{j \in \mathcal{J}_i} \sum_{n \in \mathcal{N}} y_{ijn(t-\tau_{ijI1})v} + \sum_{i \in \mathcal{I}} \sum_{j \in \mathcal{J}_i} \bigl(e_{ijI1(t-\tau_{ijI1})v} - e_{I1ij(t-1)v} \bigr), \\
    &\nonumber \quad\quad\quad\quad \forall t \in \mathcal{T} \setminus 1, \forall v \in \mathcal{V}  \\
    &c_{01tv} = c_{01(t-1)v} - \sum_{i \in \mathcal{I'}} \sum_{j \in \mathcal{J}_i} \sum_{n \in \mathcal{N}} z_{ijn(t-1)v} + \sum_{i \in \mathcal{I}} \sum_{j \in \mathcal{J}_i} \bigl(  e_{ij01(t-\tau_{ijIJ})v} - e_{01ij(t-1)v} \bigr) - \sum_{n \in \mathcal{N}} y_{01n(t-1)v}, \\
    &\nonumber \quad\quad\quad\quad \forall t \in \mathcal{T} \setminus 1, \forall v \in \mathcal{V}  \\ 
    &\sum_{i \in \mathcal{I}\setminus I} \sum_{j \in \mathcal{J}_i} \sum_{t=1}^{r_n-\tau_{ijI1}-1} \sum_{v \in \mathcal{V}} y_{ijntv} = 0, \quad\quad \forall n \in \mathcal{N} \\
    &\sum_{i \in \mathcal{I}\setminus I} \sum_{j \in \mathcal{J_i}} \sum_{t=r_n-\tau_{ijI1}}^{r_n + \rho_n - \tau_{ijI1}} \sum_{v \in \mathcal{V}} y_{ijntv} = 1, \quad\quad \forall n \in \mathcal{N} \\
    &\sum_{i \in \mathcal{I}\setminus I} \sum_{j \in \mathcal{J}_i} \sum_{t=r_n+\rho_n-\tau_{ijI1}+1}^{T} \sum_{v \in \mathcal{V}} y_{ijntv} = 0, \quad\quad \forall n \in \mathcal{N} \\
    &\sum_{i \in \mathcal{I'}} \sum_{j \in \mathcal{J}_i} \sum_{t=1}^{a_n-1} \sum_{v \in \mathcal{V}} z_{ijntv} = 0, \quad\quad \forall n \in \mathcal{N} \\
    &\sum_{i \in \mathcal{I'}} \sum_{j \in \mathcal{J}_i} \sum_{t=a_n}^{a_n + \alpha_n} \sum_{v \in \mathcal{V}} z_{ijntv} + \sum_{t=r_n-\tau_{01I1}}^{\min ( r_n+\rho_n-\tau_{01I1}, a_n + \alpha_n)} \sum_{v \in \mathcal{V}} y_{01ntv} = 1, \quad\quad \forall n \in \mathcal{N} \\
    &\sum_{i \in \mathcal{I'}} \sum_{j \in \mathcal{J}_i} \sum_{t=a_n+\alpha_n+1}^{T} \sum_{v \in \mathcal{V}} z_{ijntv} = 0, \quad\quad \forall n \in \mathcal{N} \\
    & \sum_{v \in \mathcal{V}} \Bigg[ \sum_{j \in \mathcal{J}_i} c_{ijtv} + \sum_{j \in \mathcal{J}_i} \biggl( \sum_{n \in \mathcal{N}} \sum_{t' \in \Omega^{in}_{z}} z_{ijnt'v} + \sum_{k,l} \sum_{n \in \mathcal{N}} \sum_{t' \in \Omega^{in}_{x}} x_{klijnt'v} + \sum_{k,l} \sum_{t' \in \Omega^{in}_{e}} e_{klijt'v} \biggr) \nonumber \\
    & \quad\quad  + \sum_{j \in \mathcal{J}_i} \biggl( \sum_{n \in \mathcal{N}} \sum_{t' \in \Omega^{out}_{y}} y_{ijnt'v} + \sum_{k,l} \sum_{n \in \mathcal{N}} \sum_{t' \in \Omega^{out}_{x}} x_{ijklnt'v} + \sum_{k,l} \sum_{t' \in \Omega^{out}_{e}} e_{ijklt'v} \biggr) \Bigg] \le 1, \quad\quad \forall i \in \mathcal{I}', \forall t \in \mathcal{T} \\
    & \sum_{n \in \mathcal{N}} (x_{ijklntv} + y_{ijntv} + e_{ijkltv}) \leq 1 - \sum_{n \in \mathcal{N}} b_{i(j+1)nt}, \quad\quad  \forall i \in \mathcal{I}', \forall j \in \mathcal{J}_i \setminus \{1\}, \forall k \in \mathcal{I}', \forall l \in \mathcal{J}_k, \forall t \in \mathcal{T}, \forall v \in \mathcal{V} 
\end{align}
\normalsize

%% file: sections/nphardness.tex
\section{NP-hardness of the BSRRP Problem}
\label{sec:nphard}

The Block Relocation Problem is known to be NP-hard \citep{CASERTA201296}. We establish the NP-hardness of the Buffer Storage, Retrieval, and Reshuffling Problem by providing a polynomial-time reduction from any instance of BRP to BSRRP.

\subsection{Problem Definitions}

\begin{definition}[BRP - Block Relocation Problem]
Given:
\begin{itemize}
    \item A set of containers $C = \{c_1, ..., c_N\}$
    \item A set of stacks $S = \{s_1, ..., s_M\}$ with maximum height $H$
    \item A priority function $p: C \rightarrow \{1, ..., N\}$ assigning unique priorities to containers
    \item An initial configuration function $f: C \rightarrow S \times \{1, ..., H\}$ mapping containers to positions
\end{itemize}
Find a sequence of relocations minimizing the total number of moves $k$ while retrieving containers in ascending priority order.
\end{definition}

\begin{definition}[BSRRP - Buffer Storage, Retrieval, and Reshuffling Problem]
Given:
\begin{itemize}
    \item A set of unit loads $U = \{u_1, ..., u_N\}$
    \item A buffer zone with lanes $L = \{l_1, ..., l_M\}$ of maximum depth $H$
    \item Time windows $[e_i, l_i]$ for each unit load $u_i$
    \item An initial configuration function $g: U \rightarrow L \times \{1, ..., H\}$
\end{itemize}
Find a sequence of relocations minimizing total travel distance $D$ while retrieving unit loads within their time windows.
\end{definition}

\subsection{Polynomial-Time Reduction}

We present a polynomial-time transformation $T$ that maps any instance $I_{BRP}$ of BRP to an instance $I_{BSRRP} = T(I_{BRP})$ of BSRRP.

\begin{definition}[Transformation $T$]
    Let $I_{BRP} = (C, S, p, f, H)$ be any instance of BRP. We construct $I_{BSRRP}=(U,L,W,g,H)$ as follows: 

1. \textbf{Structure Preservation:}
\begin{itemize}
    \item Create a set of unit loads $U$ such that $|U| = |C|$ (same number of elements)
    \item Create a set of Lanes $L$ such that $|L| = |S|$ (same number of stacks/lanes)
    \item For each container $c_i \in C$, create a corresponding unit load $u_i \in U$
    \item For each container position $(s,h) = f(c_i)$, set the position of the corresponding unit load $g(u_i) = (l,h)$ where $l$ is the lane corresponding to stack $s$
\end{itemize}

2. \textbf{Time Window Construction:}
Let $k_{max}(N) = \frac{N(N-1)}{2}$ be the maximum possible relocations for $N=|U|=|C|$ unit loads.

For each $u_i$ corresponding to container $c_i$ with priority $p(c_i) = i$:
\begin{itemize}
    \item $e_i = (i-1) * (k_{max}(N) + 1) + 1$
    \item $l_i = i * (k_{max}(N) + 1)$
\end{itemize}

3. \textbf{Distance Metric:}
Define the distance function $d(x, y)$ between any two positions $x, y$ in the BSRRP instance as follows:

   \[ d(x, y) = 
       \begin{cases} 
           1, & \text{if the move from } x \text{ to } y \text{ involves carrying a unit load (loaded move)} \\
           0, & \text{if the move from } x \text{ to } y \text{ does not involve carrying a unit load (unloaded move)} 
       \end{cases}
   \]
\end{definition}

This distance metric directly counts the number of loaded moves (relocations and retrievals), since unloaded moves contribute to zero distance.

\subsection{Proof of Correctness}

\begin{lemma}[Correctness of Mapping $T$]
The mapping of containers $c_i$ to unit loads $u_i$ and stacks $s_j$ to lanes $l_j$ defined by the transformation $T$ preserves all accessibility relationships between the elements.
\end{lemma}

\begin{proof}
As described in the definition of transformation $T$, the following holds:
\begin{itemize}
    \item For each $c_i \in C$, there exists a corresponding $u_i \in U$.
    \item For each $s_j \in S$, there exists a corresponding $l_j \in L$.
    \item The position of each element $u_i$ in $I_{BSRRP}$ corresponds to the position of the corresponding element $c_i$ in $I_{BRP}$ (i.e., if $f(c_i) = (s, h)$, then $g(u_i) = (l, h)$, where $l$ is the lane corresponding to $s$).
\end{itemize}

Therefore, for any containers $c_i, c_j \in C$:
\begin{itemize}
    \item If $c_i$ blocks $c_j$ in $I_{BRP}$, then $u_i$ blocks $u_j$ in $I_{BSRRP}$.
    \item If $c_i$ is accessible in $I_{BRP}$, then $u_i$ is accessible in $I_{BSRRP}$.
    \item The mapping $g$ preserves the relative positions of all elements.
\end{itemize}
Thus, any valid access and relocation sequence in one problem has a corresponding valid sequence in the other problem.
\end{proof}

\begin{lemma}[Time Window Correctness]
The construction of the time window enforces the priority ordering of BRP while allowing for all necessary relocations for each unit load in $I_{BSRRP}$.
\end{lemma}

\begin{proof}
The time windows $[e_i, l_i]$ for each unit load $u_i$ (corresponding to container $c_i$ with BRP priority $i$) are defined as $e_i = (i-1)(k_{max}(N)+1)+1$ and $l_i = i(k_{max}(N)+1)$.
This construction creates sequential and strictly nonoverlapping time windows, such as $e_i = l_{i-1} + 1$. Consequently, operations related to $u_i$ can only begin after the time window for $u_{i-1}$ has closed, thus directly enforcing the ascending priority order of the BRP.

The duration of each time window for $u_i$ is $l_i - e_i + 1 = k_{max}(N)+1$ time units. In the BSRRP instance, each unit of time corresponds to a loaded move (either a relocation or a retrieval), according to the defined distance metric. The value $k_{max}(N) = \frac{N(N-1)}{2}$ represents an upper bound on the total number of relocations required for the entire BRP instance.
Allocating $k_{max}(N)+1$ time units (i.e., potential loaded moves) for each individual unit load $u_i$ is therefore amply sufficient to accommodate:
\begin{itemize}
    \item Any relocations necessary to access $u_i$ after $u_1, \dots, u_{i-1}$ have been retrieved. The number of such relocations for $u_i$ alone will not exceed $N-1$, which is less than or equal to $k_{max}(N)$ for $N \ge 2$.
    \item The single loaded move required for the retrieval of $u_i$ itself.
\end{itemize}
Thus, any valid sequence of BRP operations can be mapped to the BSRRP instance within the constructed time windows, respecting the priority order and allowing for all required relocations and retrievals.
\end{proof}

\begin{theorem}[Solution Equivalence]
An optimal solution to $I_{BRP}$ with $k$ relocations exists if and only if an optimal solution to $I_{BSRRP}$ with total distance $D = k + N$ exists.
\end{theorem}

\begin{proof}
Let $Sol_{BRP}$ be a feasible sequence of operations for $I_{BRP}$ involving $k$ relocations and $N$ mandatory retrievals. Let $Sol_{BSRRP}$ be the corresponding sequence of operations for $I_{BSRRP}$.

Under the defined distance metric, unloaded moves in $Sol_{BSRRP}$ have a distance of 0. Loaded moves have a distance of 1.
Each of the $k$ relocations in $Sol_{BRP}$ corresponds to exactly one loaded move in $Sol_{BSRRP}$ (moving the relocated item).
Each of the $N$ retrievals in $Sol_{BRP}$ corresponds to exactly one loaded move in $Sol_{BSRRP}$ (moving the retrieved item out).

Therefore, the total distance $D$ for $Sol_{BSRRP}$ is precisely the sum of the distances of the loaded moves:
$D = (k \times 1) + (N \times 1) = k + N$.

\textbf{($\Rightarrow$) Optimal BRP Solution implies Optimal BSRRP Solution:}
Assume $Sol_{BRP}^*$ is an optimal solution to $I_{BRP}$ with $k^*$ relocations. The corresponding $Sol_{BSRRP}^*$ has total distance $D^* = k^* + N$.
Suppose, for contradiction, that $Sol_{BSRRP}^*$ is not optimal for $I_{BSRRP}$. Then there must exist a different feasible solution $Sol'_{BSRRP}$ with total distance $D' < D^*$. Since distance only counts loaded moves, $D'$ must correspond to some number of relocations $k'$ and the $N$ retrievals, such that $D' = k' + N$.
Given $D' < D^*$, we have $k' + N < k^* + N$, which implies $k' < k^*$. The sequence $Sol'_{BSRRP}$ corresponds to a valid sequence $Sol'_{BRP}$ with $k'$ relocations (due to Lemma 4.1 and Lemma 4.2). This contradicts the assumption that $Sol_{BRP}^*$ with $k^*$ relocations was optimal for $I_{BRP}$. Therefore, $Sol_{BSRRP}^*$ must be optimal for $I_{BSRRP}$.

\textbf{($\Leftarrow$) Optimal BSRRP Solution implies Optimal BRP Solution:}
Assume $Sol_{BSRRP}^*$ is an optimal solution to $I_{BSRRP}$ with total distance $D^*$. As shown above, $D^*$ must be equal to the number of loaded moves, which corresponds to $k^*$ relocations and $N$ retrievals, so $D^* = k^* + N$. This solution $Sol_{BSRRP}^*$ corresponds to a valid BRP solution $Sol_{BRP}^*$ with $k^* = D^* - N$ relocations.
Suppose, for contradiction, that $Sol_{BRP}^*$ is not optimal for $I_{BRP}$. Then there must exist a different feasible solution $Sol'_{BRP}$ with $k'$ relocations such that $k' < k^*$. From the ($\Rightarrow$) direction, this $Sol'_{BRP}$ corresponds to a BSRRP solution $Sol'_{BSRRP}$ with total distance $D' = k' + N$.
Since $k' < k^*$, it follows that $D' < D^*$. This contradicts the assumption that $Sol_{BSRRP}^*$ with distance $D^*$ was optimal for $I_{BSRRP}$. Therefore, $Sol_{BRP}^*$ must be optimal for $I_{BRP}$.

Thus, an optimal solution with $k$ relocations exists for $I_{BRP}$ if and only if an optimal solution with total distance $D = k + N$ exists for $I_{BSRRP}$.
\end{proof}

\subsection{Polynomial-Time Reduction Complexity}
The transformation $T$ operates in polynomial time with respect to the input size (primarily $N$):
\begin{itemize}
    \item Configuration mapping: $O(N)$
    \item Time window computation: The loop runs $N$ times with constant time arithmetic operations inside. $O(N)$.
    \item Distance metric definition: $O(1)$
\end{itemize}
Thus, the transformation $T$ is polynomial-time.

\subsection{Conclusion}

Since the reduction is polynomial-time, while it preserves solution feasibility and optimality for all instances, and the BRP is proven NP-hard, the BSRRP is also NP-hard. $\square$